\documentclass{article}
\usepackage{authblk}
\usepackage[utf8]{inputenc} % allow utf-8 input
\usepackage[T1]{fontenc}    % use 8-bit T1 fonts
\usepackage[a4paper, margin=2.7cm]{geometry}
\usepackage{hyperref}       % hyperlinks
\usepackage{url}            % simple URL typesetting
\usepackage{booktabs}       % professional-quality tables
\usepackage{amsfonts}       % blackboard math symbols
\usepackage{nicefrac}       % compact symbols for 1/2, etc.
\usepackage{microtype}      % microtypography
\usepackage{xcolor}         % colors

\usepackage{caption}
\usepackage{subcaption}
\usepackage{graphicx}
\usepackage{geometry}
\usepackage{enumerate}
\usepackage{amsmath}
\usepackage{amsthm}
\usepackage{bm}
\usepackage{booktabs}
\usepackage{adjustbox}
\usepackage{algorithm}
\usepackage{algpseudocode}
\newtheorem*{remark}{Remark}
\usepackage{adjustbox}
\usepackage{tabularx}

\newtheorem{Definition}{Definition}

\newtheorem{Theorem}{Theorem}
\newtheorem{Lemma}{Lemma}
\newtheorem{Assumption}{Assumption}

\newtheorem{Example}{Example}

\date{29 Nov 2022}

\begin{document}

\title{Counterfactual Supervision-based Information Bottleneck for Out-of-Distribution Generalization}

\author[ ]{Bin Deng}
\author[ ]{Kui Jia}
\affil[ ]{South China University of Technology}
\affil[ ]{\textit{eebindeng@mail.scut.edu.cn}}

\maketitle
\begin{abstract}
 Learning invariant (causal) features for out-of-distribution (OOD) generalization has attracted extensive attention recently, and among the proposals invariant risk minimization (IRM) is a notable solution. In spite of its theoretical promise for linear regression, the challenges of using IRM in linear classification problems remain. By introducing the information bottleneck (IB) principle into the learning of IRM, IB-IRM approach has demonstrated its power to solve these challenges. In this paper, we further improve IB-IRM from two aspects. First, we show that the key assumption of support overlap of invariant features used in IB-IRM is strong for the guarantee of OOD generalization and it is still possible to achieve the optimal solution without this assumption. Second, we illustrate two failure modes that IB-IRM (and IRM) could fail for learning the invariant features, and to address such failures, we propose a \textit{Counterfactual Supervision-based Information Bottleneck (CSIB)} learning algorithm that provably recovers the invariant features. By requiring counterfactual inference, CSIB works even when accessing data from a single environment. Empirical experiments on several datasets verify our theoretical results.
\end{abstract}

\section{Introduction}\label{section-introduction}
Modern machine learning models are prone to catastrophic performance loss during deployment when the test distribution is different from the training distribution. This phenomenon has been repeatedly witnessed and intentionally exposed in many examples \cite{szegedy2013intriguing, rosenfeld2018elephant, geirhos2018imagenet, nguyen2015deep, gururangan2018annotation}. Among the explanations, shortcut learning \cite{geirhos2020shortcut} is considered as a main factor causing this phenomenon. A nice example is about the classification of images of cows and camels --- a trained convolutional network tends to recognize cows or camels by learning spurious features from image backgrounds (e.g., green pastures for cows and deserts for camels), rather than learning the causal shape features of the animals \cite{beery2018recognition}; decisions based on the spurious features would make the learned models fail when cows or camels appear in unusual, different environments. Machine learning models are expected to have the capability of out-of-distribution (OOD) generalization and avoid shortcut learning.

To achieve OOD generalization, recent theories \cite{arjovsky2019invariant,krueger2021out,ahuja2020invariant,pezeshki2021gradient,ahuja2021invariance} are motivated by causality literature \cite{pearl2009causality, peters2017elements}, and resort to extraction of the invariant, causal features and establishing the relevant conditions under which machine learning models have the guaranteed generalization. Among these works, invariant risk minimization (IRM) \cite{arjovsky2019invariant} is a notable learning paradigm that incorporates the invariance principle \cite{peters2016causal} into practice.
In spite of the theoretical promise of IRM, it is only applicable to problems of linear regression. For other problems such as linear classification, Ahuja et al. \cite{ahuja2021invariance} first show that for OOD generalization, linear classification is more difficult (see Theorem \ref{theorem-impossibility}), and propose a new learning method of information bottleneck-based invariant risk minimization (IB-IRM) based on the support overlap assumption (Assumption \ref{assp-inv_feature_overlap}). In this work, we closely investigate the conditions identified in \cite{ahuja2021invariance} and propose improved results for OOD generalization of linear classification. Our technical contributions are as follows.

\textbf{Contributions.} In \cite{ahuja2021invariance}, a notion of support overlap of invariant features is assumed in order to make the OOD generalization of linear classification successful. In this work, we first show that this assumption is strong and it is still possible to achieve such goal without this assumption. Then, we examine whether the IB-IRM proposed in \cite{ahuja2021invariance} is sufficient to learn invariant features for linear classification, and find that IB-IRM (and IRM) could fail in two modes. We then analyze two failure modes of IB-IRM and IRM, in particular when the spurious features in training environments capture sufficient information for the task of interest but have less information than the invariant features. Based on the above analyses, we propose a new method, termed counterfactual supervision-based information bottleneck (CSIB), to address such failures. We prove that, without the need of the support overlap assumption, CSIB is theoretically guaranteed for the success of OOD generalization in linear classification. Notably, CSIB works even when accessing data from a single environment. Finally, we design three synthetic datasets and a colored minst dataset based on our used motivating examples; experiments demonstrate the effectiveness of CSIB empirically.

The rest of this article is organized as follows. The learning problem of out-of-distribution (OOD) generalization is formulated in Section \ref{section-ood-formulations}. In Section \ref{section-assumptions}, we study the learnability of the OOD generalization with different assumptions to the training and test environments. Using these assumptions, two failure modes of previous methods (IRM and IB-IRM) are analysed in Section \ref{section-failures}. Based on the above analysis, our method is then proposed in Section \ref{section-method}. The experiments are reported in Section \ref{section-experiments}. Finally, we discuss the related works in Section \ref{section-related-work} and provide some conclusions and limitations of our work in Section \ref{section-conclusion-and-limitation}.
All the proofs and details of experiments are given in the appendices.

\section{OOD generalization: background and formulations}\label{section-ood-formulations}

\subsection{Background on structural equation models}
Before introducing our formulations of OOD generalization, we provide a detailed background on structural equation models (SEMs) \cite{pearl2009causality, arjovsky2019invariant}.
\begin{Definition}[Structural Equation Model (SEM)]
A structural equation model (SEM) $\mathcal{C}:=(\mathcal{S},N)$ governing the random vector $X = (X_1,...,X_d)$ is a set of structural equations:
\begin{equation*}
    \mathcal{S}_i: X_i \leftarrow f_i(Pa(X_i), N_i),
\end{equation*}
where $Pa(X_i)\subseteq \{X_1,...,X_d\}\setminus \{X_i\}$ are called the parents of $X_i$, and $N_i$ are independent noise random variables. For every SEM, we yield a directed acyclic graph (DAG) $\mathcal{G}$ by adding one vertex for each $X_i$ and directed edges from each parent in $Pa(X_i)$ (the causes) to child $X_i$ (the effect).
\end{Definition}

\begin{Definition}[Intervention]\label{def-intervention}
    Consider a SEM $\mathcal{C} = (\mathcal{S}, N)$. An intervention $e$ on $\mathcal{C}$ consists of replacing one or several of its structural equations to obtain an intervened SEM $\mathcal{C}^e = (\mathcal{S}^e,N^e)$, with structural equations:
    \begin{equation*}
        \mathcal{S}^e_i: X^e_i \leftarrow f^e_i(Pa^e(X^e_i), N^e_i),
    \end{equation*}
    The variable $X^e$ is intervened if $\mathcal{S}_i \neq \mathcal{S}^e_i$ or $N_i \neq N^e_i$.
\end{Definition}

In a SEM $\mathcal{C}$, we can draw samples from the observational distribution $\mathbb{P}(X)$ according to the topological ordering of its DAG $\mathcal{G}$. We also can manipulate (intervene) an unique SEM $\mathcal{C}$ in different ways, indexed by $e$, to different but related SEMs $\mathcal{C}^e$, which results in different interventional distributions $\mathbb{P}(X^e)$. Such family of interventions are used to model the environments.

\subsection{Formulations of OOD generalization}
In this paper, we study the OOD generalization problem by following the linear classification structural equation model in below \cite{ahuja2021invariance}.
\begin{Assumption}[Linear classification SEM $\mathcal{C}_{ood}$]\label{assp-linearSEM}
\begin{equation}
\begin{aligned}
& Y \leftarrow \bm{1}(w^*_{inv}\cdot Z_{inv}) \oplus N, \quad N\sim Bernoulli(q), \ q < \frac{1}{2};\\
& X \leftarrow S(Z_{inv}, Z_{spu}),
 \end{aligned}
\end{equation}
where $w^*_{inv} \in \mathbb{R}^m$ is the labeling hyperplane, $Z_{inv}\in \mathbb{R}^m$, $Z_{spu}\in \mathbb{R}^o$, $X\in\mathbb{R}^d$, $\oplus$ is the XOR operator, $S\in\mathbb{R}^{d\times (m+o)}$ is invertible ($d = m+o$), $\cdot$ is the dot product function, and $\bm{1}(a) = 1$ if $a\geq 0$ otherwise $0$.
\end{Assumption}

\begin{figure}
% \begin{adjustwidth}{-\extralength}{0cm}
  \centering
  \begin{subfigure}[b]{0.24\textwidth}
    \centering
    \includegraphics[height=1.0in]{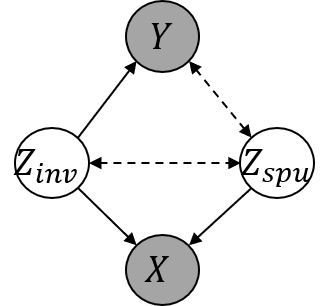}
    \caption{}\label{fig-DAG0}
  \end{subfigure}
  % \qquad
  \begin{subfigure}[b]{0.24\textwidth}
    \centering
    \includegraphics[height=1.0in]{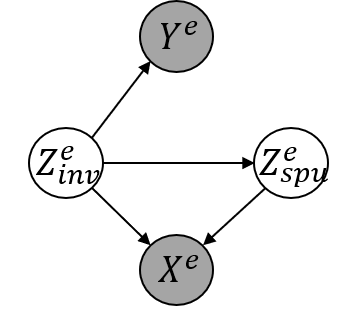}
    \caption{}\label{fig-DAG1}
  \end{subfigure}
  % \qquad
  \begin{subfigure}[b]{0.24\textwidth}
    \centering
    \includegraphics[height=1.0in]{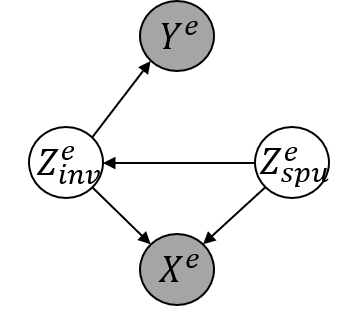}
    \caption{}\label{fig-DAG2}
  \end{subfigure}
  % \qquad
  \begin{subfigure}[b]{0.24\textwidth}
    \centering
    \includegraphics[height=1.0in]{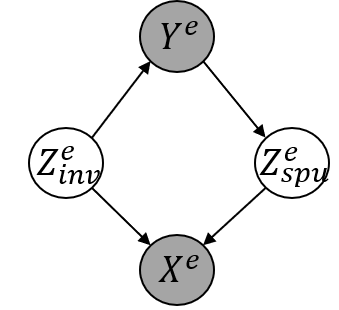}
    \caption{}\label{fig-DAG3}
  \end{subfigure}
% \end{adjustwidth}
  \caption{(a) DAG of the SEM $\mathcal{C}_{ood}$ (Assumption \ref{assp-linearSEM}). (b) - (d) DAGs of the interventional SEM $\mathcal{C}^e_{ood}$ in the training environments $\mathcal{E}_{tr}$ with respect to different correlations between $Z_{inv}$ and $Z_{spu}$. Grey nodes denote observed variables and white nodes represent unobserved variables. Dashed lines denote the edges which might vary across the interventional environments and even be absent in some scenarios, whilst solid lines indicate that they are invariant across all the environments. All exogenous noise variable are omitted in the DAGs.}
\end{figure}

The SEM $\mathcal{C}_{ood}$ governs four random variables $\{X, Y, Z_{inv}, Z_{spu}\}$ and its directed acyclic graph (DAG) is illustrated in Figure \ref{fig-DAG0}, where the exogenous noise variable $N$ is omitted.
Following Definition \ref{def-intervention}, each intervention $e$ generates a new environment $e$ with interventional distribution $\mathbb{P}(X^e,Y^e,Z^e_{inv}, Z^e_{spu})$. We assume only the variables of $X^e$ and $Y^e$ are observable. In OOD generalization, we are interested in a set of environments $\mathcal{E}_{all}$ defined as below.

\begin{Definition}[$\mathcal{E}_{all}$] \label{def-environments}
    Consider the SEM $\mathcal{C}_{ood}$ (Assumption \ref{assp-linearSEM}) and the learning goal of predicting $Y$ from $X$. Then, the set of all environments $\mathcal{E}_{all}(\mathcal{C}_{ood})$ indexes all the interventional distributions $\mathbb{P}(X^e, Y^e)$ obtainable by valid interventions $e$. An intervention $e\in \mathcal{E}_{all}(\mathcal{C}_{ood})$ is valid as long as (i) the DAG remains acyclic, (ii) $\mathbb{P}(Y^e|Z^e_{inv}) = \mathbb{P}(Y|Z_{inv})$, and (iii) $\mathbb{P}(X^e|Z^e_{inv}, Z^e_{spu}) = \mathbb{P}(X|Z_{inv}, Z_{spu})$.
\end{Definition}

The Assumption \ref{assp-linearSEM} shows that $Z_{inv}$ is the cause of the response $Y$. We name $Z_{inv}$ the invariant features or causal features because $\mathbb{P}(Y^e|Z^e_{inv}) = \mathbb{P}(Y|Z_{inv})$ always holds among all valid interventional SEMs $\mathcal{C}^e_{ood}$, as defined in Definition \ref{def-environments}. The $Z_{spu}$ is called spurious features, because $\mathbb{P}(Y^e|Z^e_{spu})$ may vary in different environments of $\mathcal{E}_{all}$.

Let $D=\{D^e\}_{e\in\mathcal{E}_{tr}}$ be the training data gathered from a set of training environments $\mathcal{E}_{tr}\subset \mathcal{E}_{all}$, where $D^e = \{(x^e_i, y^e_i)\}^{n_e}_{i=1}$ is the dataset from environment $e$ with each instance $(x^e_i, y^e_i)$ i.i.d. drawn from $\mathbb{P}(X^e, Y^e)$. Let $\mathcal{X}^e \subseteq \mathbb{R}^d$ and $\mathcal{Y}\subseteq \{0,1\}$ be the support sets of $X^e$ and $Y$, respectively. Given observed data $D$, the goal of OOD generalization is to find a predictor $f: \mathbb{R}^d \rightarrow \mathcal{Y}$ such that it can perform well across a set of OOD environments (test environments) $\mathcal{E}_{ood}$ of interest, where $\mathcal{E}_{ood} \subseteq \mathcal{E}_{all}$. Formally, it is expected to minimize
\begin{equation}\label{equ-ood}
\max_{e\in \mathcal{E}_{ood}} R^e(f),
\end{equation}
where $R^e(f):=\mathbb{E}_{X^e, Y^e}[l(f(X^e), Y^e)]$ is the risk under the environment $e$ with $l(\cdot, \cdot)$ the 0-1 loss function. Since $\mathcal{E}_{ood}$ may be different from $\mathcal{E}_{tr}$, this learning problem is called OOD generalization.
% Clearly, without any restrictions on $\mathcal{E}_{all}$, it is impossible to achieve OOD generalization. The definition of $\mathcal{E}_{all}$ is as follows.
We assume the predictor $f = w\circ \Phi$ includes a feature extractor $\Phi: \mathcal{X}\rightarrow \mathcal{H}$ and a classifier $w: \mathcal{H}\rightarrow \mathcal{Y}$. With a slight abuse of notation, we also let the classifier $w$ and feature extractor $\Phi$ be parameteried by themselves respectively as $w\in\mathbb{R}^{c+1}$ and $\Phi\in\mathbb{R}^{c\times d}$ with $c$ the number of feature dimension.

\subsection{Background on IRM and IB-IRM}
To minimize Equation (\ref{equ-ood}), two notable solutions of IRM \cite{arjovsky2019invariant} and IB-IRM \cite{ahuja2021invariance} are listed as follows:
\begin{equation}\label{equ-IRM}
\text{IRM:} \quad \min_{w,\Phi} \frac{1}{|\mathcal{E}_{tr}|}\sum_{e\in\mathcal{E}_{tr}} R^e(w\circ\Phi), \ \text{s.t.} \ w\in \arg \min_{\tilde{w}} R^e(\tilde{w}\circ \Phi), \forall e \in \mathcal{E}_{tr},
\end{equation}
\begin{equation}\label{equ-IBIRM}
\text{IB-IRM:} \quad \min_{w, \Phi} \sum_{e\in\mathcal{E}_{tr}} h^e(\Phi), \ \text{s.t.} \ \frac{1}{|\mathcal{E}_{tr}|}\sum_{e\in \mathcal{E}_{tr}} R^e(w\circ\Phi) \leq r^{th}, w\in \arg \min_{\tilde{w}} R^e(\tilde{w}\circ \Phi), \forall e\in \mathcal{E}_{tr},
\end{equation}
where $R^e(w\circ\Phi)=\mathbb{E}_{X^e, Y^e}[l(w\circ\Phi(X^e), Y^e)]$, and $h^e(\Phi) = H(\Phi(X^e))$ with $H$ the Shannon entropy (or a lower bounded differential entropy) and $r^{th}$ is the threshold on the average risk. If we drop the invariance constraint from IRM and IB-IRM, we get standard empirical risk minimization (ERM) and information bottleneck-based empirical risk minimization (IB-ERM) respectively. The use of entropy constraint in IB-IRM is inspired from the information bottleneck principle \cite{tishby1999information} where mutual information $I(X; \Phi(X))$ is used for information compression. Since the representation $\Phi(X)$ is a deterministic mapping of $X$, we have
\begin{equation}
    I(X;\Phi(X)) = H(\Phi(X)) - H(\Phi(X)|X) = H(\Phi(X)),
\end{equation}
thus minimizing the entropy of $\Phi(X)$ is equivalent to minimizing the mutual information $I(X; \Phi(X))$. In brief, the optimization goal of IB-IRM is to select the one that has the least entropy among all highly predictive invariant predictors.

\section{OOD generalization: assumptions and learnability}\label{section-assumptions}
To study the learnability of OOD generalization, we make following definition.
\begin{Definition}
Given $\mathcal{E}_{tr}\subset \mathcal{E}_{all}$ and $\mathcal{E}_{ood}\subseteq \mathcal{E}_{all}$. We say an algorithm succeeds to solve OOD generalization with respect to ($\mathcal{E}_{tr}, \mathcal{E}_{ood}$) if the predictor $f^*\in \mathcal{F}$ returned by this algorithm satisfies the following equation:
\begin{equation}
    \max_{e\in \mathcal{E}_{ood}} R^e(f^*) = \min_{f\in\mathcal{F}}\max_{e\in \mathcal{E}_{ood}} R^e(f),
\end{equation}
where $\mathcal{F}$ is the learning hypothesis (a function set including all possible linear classifier).
Otherwise we say it fails to solve OOD generalization.
\end{Definition}

\begin{figure}
% \begin{adjustwidth}{-\extralength}{0cm}
  % \centering
  % \begin{subfigure}[b]{0.23\textwidth}
  %   \includegraphics[width=1.0\textwidth]{images/Example1-1.png}
  %   \caption{}\label{fig-example1}
  % \end{subfigure}
  % ~~
  \centering
  \begin{subfigure}[b]{0.3\textwidth}
    \includegraphics[width=1.0\textwidth]{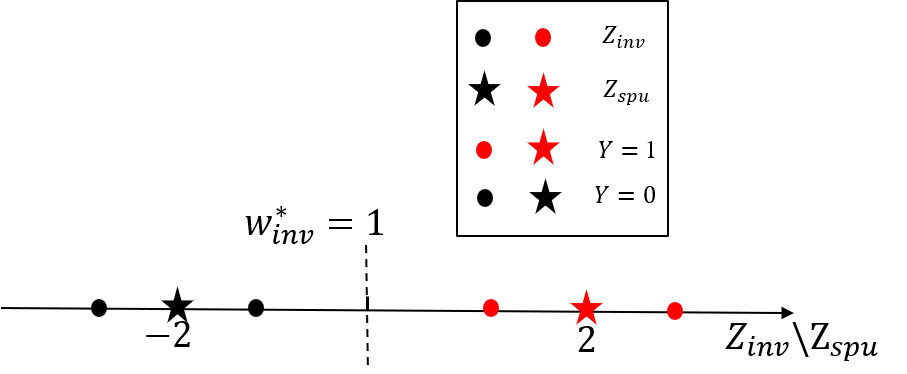}
    \caption{}\label{fig-example1}
  \end{subfigure}
  ~~
  \begin{subfigure}[b]{0.3\textwidth}
    \includegraphics[width=1.0\textwidth]{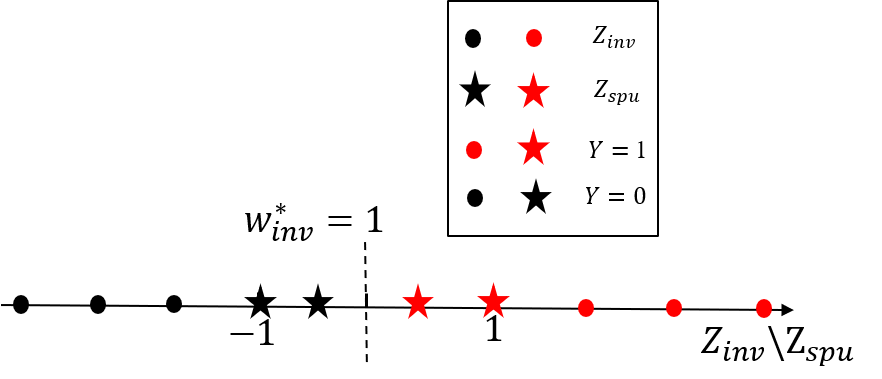}
    \caption{}\label{fig-example2}
  \end{subfigure}
  ~~
  \begin{subfigure}[b]{0.3\textwidth}
    \includegraphics[width=1.0\textwidth]{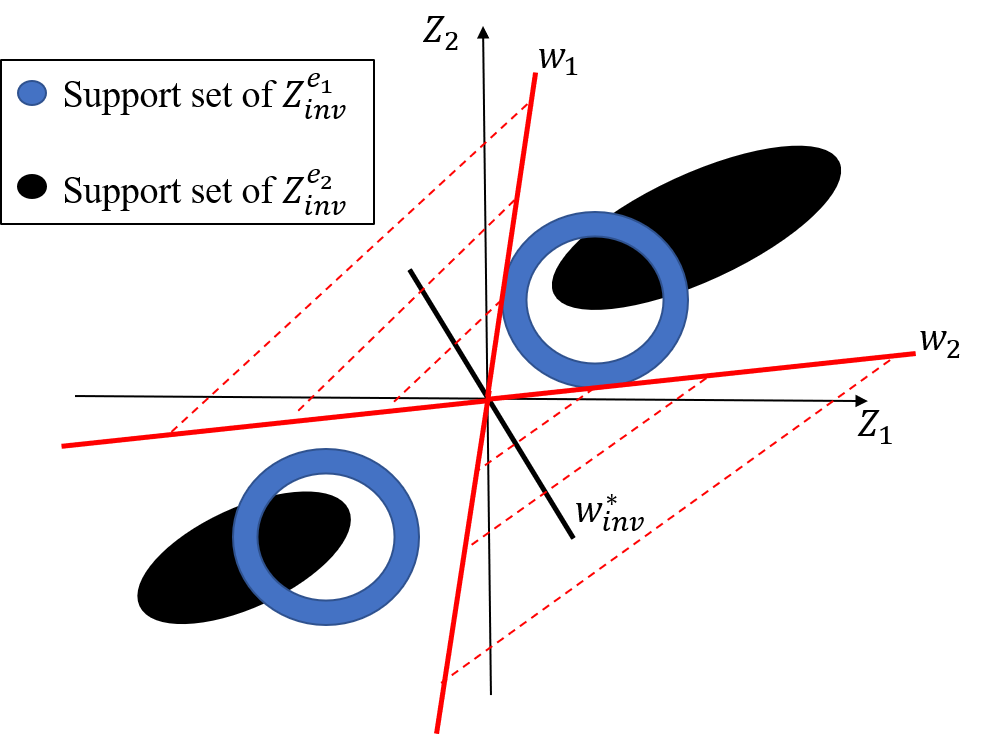}
    \caption{}\label{fig-example-supprtset}
  \end{subfigure}
% \end{adjustwidth}
  \caption{(a) Example \ref{example-1}. (b) Example \ref{example-2}. (c) Example illustration. Here, $dim(Z_{inv})=2$ and $Z_{inv} = (Z_1, Z_2)$. The blue and black regions represent the support sets of $Z^{e_1}_{inv}$ and $Z^{e_2}_{inv}$, corresponding to the environments $e_1$ and $e_2$ respectively. $\mathcal{E}_{tr} = \{e_1\}$ is the training environment and $\mathcal{E}_{ood} = \{e_2\}$ is the OOD environment. Although Assumption \ref{assp-inv_feature_overlap} does not hold in this example, any zero-error classifier with $\Phi(X) = Z_{inv}$ on the $e_1$ environment data would clearly make the classification error to be zero in $e_2$, thus succeeds to solve OOD generalization.}
\end{figure}

So far, we have omitted how different environments of $\mathcal{E}_{tr}$ and $\mathcal{E}_{ood}$ exactly are to enable OOD generalization. Different assumptions about $\mathcal{E}_{tr}$ and $\mathcal{E}_{ood}$ make the OOD generalization problem different.

\subsection{Assumptions about the training environments $\mathcal{E}_{tr}$}
Define the support set of the invariant (\emph{resp.}, spurious) features $Z^e_{inv}$ (\emph{resp.}, $Z^e_{spu}$) in environment $e$ as $\mathcal{Z}^e_{inv}$ (\emph{resp.}, $\mathcal{Z}^e_{spu}$). In general, we make following assumptions to the invariant features $\mathcal{Z}^e_{inv}$ in the training environments $\mathcal{E}_{tr}$.

\begin{Assumption}[Bounded invariant features]\label{assp-bound_inv_feature}
$\cup_{e\in\mathcal{E}_{tr}} \mathcal{Z}^e_{inv}$ is a bounded set\footnote{A set $\mathcal{Z}$ is bounded if $\exists M<\infty$ such that $\forall z\in\mathcal{Z}, \|z\|\leq M$.}.
\end{Assumption}
% \begin{Assumption}[Bounded spurious features]\label{assp-bound_spu_feature}
% $\cup_{e\in\mathcal{E}_{tr}} \mathcal{Z}^e_{spu}$ is a bounded set.
% \end{Assumption}
\begin{Assumption}[Strictly separable invariant features]\label{assp-sep_inv_feature}
$\forall z\in\cup_{e\in\mathcal{E}_{tr}}\mathcal{Z}^e_{inv}, w^*_{inv}\cdot z \neq 0.$
\end{Assumption}

The difficulties of OOD generalization is due to the spurious correlations between $Z_{inv}$ and $Z_{spu}$ in the training environments $\mathcal{E}_{tr}$. In this paper, we consider three modes induced by different correlations between $Z_{inv}$ and $Z_{spu}$ as shown below.
\begin{Assumption}[Spurious correlation 1] \label{assp-mode-1}
Assume each $e\in \mathcal{E}_{tr}$,
\begin{equation}
\begin{aligned}
% & Y^e \leftarrow \bm{1}(w^*_{inv}\cdot Z^e_{inv}) \oplus N, \quad N\sim Bernoulli(q), \ q < \frac{1}{2};\\
& Z^e_{spu} \leftarrow AZ^e_{inv} + W^e; \\
% & X^e \leftarrow S(Z^e_{inv}, Z^e_{spu}),
 \end{aligned}
\end{equation}
where, $A\in \mathbb{R}^{o\times m}$, and $W^e\in\mathbb{R}^o$ is a continuous (or discrete with each component supported on at least two distinct values), bounded, and zero mean noise variable.
% , and other notations follow by the definition in Assumption \ref{assp-linearSEM}.
\end{Assumption}

\begin{Assumption}[Spurious correlation 2] \label{assp-mode-2}
Assume each $e\in \mathcal{E}_{tr}$,
\begin{equation}
\begin{aligned}
% & Y^e \leftarrow \bm{1}(w^*_{inv}\cdot Z^e_{inv}) \oplus N, \quad N\sim Bernoulli(q), \ q < \frac{1}{2};\\
& Z^e_{inv} \leftarrow AZ^e_{spu} + W^e; \\
% & X^e \leftarrow S(Z^e_{inv}, Z^e_{spu}),
 \end{aligned}
\end{equation}
where, $A\in \mathbb{R}^{m\times o}$, and $W^e\in\mathbb{R}^m$ is a continuous (or discrete with each component supported on at least two distinct values), bounded, and zero mean noise variable.
% , and other notations follow by the definition in Assumption \ref{assp-linearSEM}.
\end{Assumption}

\begin{Assumption}[Spurious correlation 3] \label{assp-mode-3}
Assume each $e\in \mathcal{E}_{tr}$,
\begin{equation}
\begin{aligned}
% & Y^e \leftarrow \bm{1}(w^*_{inv}\cdot Z^e_{inv}) \oplus N, \quad N\sim Bernoulli(q), \ q < \frac{1}{2};\\
& Z^e_{spu} \leftarrow W_1^e Y^e + W_0^e (1-Y^e); \\
% & X^e \leftarrow S(Z^e_{inv}, Z^e_{spu}),
 \end{aligned}
\end{equation}
where $W_0^e\in\mathbb{R}^o$ and $W_1^e\in\mathbb{R}^o$ are independent noise variables.
% , and other notations follow by the definition in Assumption \ref{assp-linearSEM}.
\end{Assumption}

For each $e\in\mathcal{E}_{tr}$, the DAGs of its corresponding interventional SEMs $\mathcal{C}^e_{ood}$ with respect to Assumptions \ref{assp-mode-1}, \ref{assp-mode-2}, and \ref{assp-mode-3} are illustrated in Figures \ref{fig-DAG1}, \ref{fig-DAG2}, and \ref{fig-DAG3}, respectively. It is worth to note that although the DAGs are identical across all training environments in each mode of Assumptions \ref{assp-mode-1}, \ref{assp-mode-2}, and \ref{assp-mode-3}, the interventional SEMs $\mathcal{C}^e_{ood}$ among different training environments are different due to the interventions on the exogenous noise variables.

% Under the above assumptions, Ahuja et al. \cite{ahuja2021invariance} shows the following theorem:
\subsection{Assumptions about the OOD environments $\mathcal{E}_{ood}$}
\begin{Theorem}[Impossibility of guaranteed OOD generalization for linear classification \cite{ahuja2021invariance}]\label{theorem-impossibility}
Suppose $\mathcal{E}_{ood} = \mathcal{E}_{all}$. If for all the training environments $\mathcal{E}_{tr}$, the latent invariant features are bounded and strictly separable, i.e., Assumptions \ref{assp-bound_inv_feature} and \ref{assp-sep_inv_feature} hold, then every deterministic algorithm fails to solve the OOD generalization.
\end{Theorem}
Above theorem shows that it is impossible to solve OOD generalization if $\mathcal{E}_{ood} = \mathcal{E}_{all}$. To make it learnable, Ahuja et al. \cite{ahuja2021invariance} propose the support overlap assumption (Assumption \ref{assp-inv_feature_overlap}) to the invariant features.
\begin{Assumption}[Invariant feature support overlap]\label{assp-inv_feature_overlap}
$\forall e \in \mathcal{E}_{ood}, \mathcal{Z}^e_{inv}\subseteq \cup_{e'\in\mathcal{E}_{tr}}\mathcal{Z}^{e'}_{inv}$.
\end{Assumption}
However, Assumption \ref{assp-inv_feature_overlap} is strong,
% For example, consider a generalization task without spurious features (the spurious features are fixed in all environments), if we have a training distribution supported on a set of $n$ samples: $D_{tr}=\{(x_i,y_i)\}_{i=1}^n$, requiring the support overlap between training data $D_{tr}$ and test data $D_{te}$ means that any test sample $(x,y)$ belongs to $D_{tr}$, i.e., for any $(x,y)\in D_{te}$, we have $(x,y)\in D_{tr}$. Therefore, the learning algorithm only needs to memorize the training data.
% For example, consider a generalization task in a single environment without spurious features, requiring the assumption of support overlap between training data and test data would promote the algorithm for memorization instead of generalization, which is a little trivial to our topic.
% With such concerns, we propose a weaker assumption that is highly related to the learning hypothesis space and connects the traditional generalization theory \cite{shalev2014understanding}.
and we would show that it is still possible to solve OOD generalization without this assumption. For better illustration, consider a OOD generalization task from $\mathbb{P}(X^{e_1}, Y^{e_2})$ to $\mathbb{P}(X^{e_2}, Y^{e_2})$ with $\mathcal{E}_{tr} = \{e_1\}$ and $\mathcal{E}_{ood}=\{e_2\}$, and the support sets of the corresponding invariant features $Z^{e_1}_{inv}$ and $Z^{e_2}_{inv}$ are intuitively illustrated in Figure \ref{fig-example-supprtset} (assume $dim(Z_{inv})=2$ in this example). From the Figure \ref{fig-example-supprtset}, it is clear that although the support sets of invariant features between the two environments are different, it is still possible to solve OOD generalization if the learned feature extractor $\Phi$ only captures the invariant features, e.g., $\Phi(X)=Z_{inv}$.

To make Assumption \ref{assp-inv_feature_overlap} weaker, we propose the following assumption.
\begin{Assumption}\label{assp-new-assp} Let $\mathbb{P}(Z_{inv}^{tr}, Y^{tr}) = \frac{1}{|\mathcal{E}_{tr}|}\sum_{e\in\mathcal{E}_{tr}}\mathbb{P}(Z^e_{inv}, Y^e)$ be the mixture distribution of invariant features in the training environments. Denote $\mathcal{A}$ be a hypothesis set including all linear classifiers mapping from $\mathbb{R}^m$ to $\mathcal{Y}$. $\forall e\in \mathcal{E}_{ood}$, assume $F_l(\mathbb{P}(Z_{inv}^{tr}, Y^{tr}))\subseteq F_l(\mathbb{P}(Z^e_{inv}, Y^e))$, where $l$ is the 0-1 loss function and $F_l(\mathbb{P}(Z, Y)) = \arg \min_{f\in \mathcal{A}} \mathbb{E}_{Z, Y} [l(f(Z),Y)]$.
\end{Assumption}
% Assumption \ref{assp-new-assp} tells us that any optimized learner by relying on invariant features only could generalize to unseen distributions.
Clearly, under the assumption of separable invariant features (Assumption \ref{assp-sep_inv_feature}), for any $e\in\mathcal{E}_{ood}$, Assumption \ref{assp-inv_feature_overlap} holds $\Rightarrow$ $\mathcal{Z}^e_{inv} \subseteq \mathcal{Z}_{inv}^{tr}$ $\Rightarrow$ $F_l(\mathbb{P}(Z_{inv}^{tr}, Y^{tr}))\subseteq F_l(\mathbb{P}(Z^e_{inv}, Y^e))$ $\Rightarrow$ Assumption \ref{assp-new-assp} holds, but not vice versa. Therefore, Assumption \ref{assp-new-assp} is weaker than Assumption \ref{assp-inv_feature_overlap}.
We would show that Assumption \ref{assp-new-assp} could be substituted for Assumption \ref{assp-inv_feature_overlap} for the success of OOD generalization in our proposed method in Section \ref{section-method}.

% This is because, if Assumption \ref{assp-inv_feature_overlap} holds, .
% {\color{red} XXXXX}

% Before that, we review another main result presented in \cite{ahuja2021invariance}.
% \begin{Theorem}[IB-IRM and IB-ERM vs. IRM and ERM \cite{ahuja2021invariance}]\label{theorem-IBIRM}
% Suppose each $e\in \mathcal{E}_{all}$ follows Assumption \ref{assp-linearSEM}. Assume that the invariant features are strictly separable, bounded, and satisfy support overlap, i.e., Assumptions \ref{assp-bound_inv_feature}, \ref{assp-inv_feature_overlap}, and \ref{assp-sep_inv_feature} hold. Also, for each $e\in\mathcal{E}_{tr}, Z^e_{spu} \leftarrow AZ^e_{inv} + W^e$, where $A \in \mathbb{R}^{o\times m}$, $W^e\in\mathbb{R}^o$ is continuous (or discrete with each component supported on at least two distinct values), bounded, and zero mean noise. Each solution to IB-IRM (Equation (\ref{equ-IBIRM}) with $l$ as 0-1 loss and $r^{th}=q$), and IB-ERM solves the OOD generalization (Equation (\ref{equ-ood})) but ERM and IRM (Equation (\ref{equ-IRM})) fail.
% \end{Theorem}

\section{Failures of IRM $\&$ IB-IRM}\label{section-failures}
Under the Spurious correlation 1 (Assumption \ref{assp-mode-1}), IB-IRM algorithm has been shown to enable OOD generalization, while IRM fails \cite{ahuja2021invariance}. In this section, we would show that both IRM and IB-IRM could fail under the Spurious correlations 2 and 3 (Assumptions \ref{assp-mode-2} and \ref{assp-mode-3}).
\subsection{Failure under the Spurious correlation 2}
\begin{Example}[Counter-Example 1]\label{example-1}
Under Assumption \ref{assp-mode-2}, let $Z^e_{inv} \leftarrow Z^e_{spu} + W^e$ with $dim(Z^e_{inv})=dim(Z^e_{spu})=dim(W^e)=1$ and $w^*_{inv}=1$ be the generated classifier in Assumption \ref{assp-linearSEM}. We assume two training environments and a OOD environment being as:
\begin{gather*}
\mathcal{E}_{tr} = \{e_1, e_2\};\quad \mathcal{E}_{ood} = \{e_3\};\\
e_1: \mathbb{P}(Z^{e_1}_{spu} = -2)=1, \mathbb{P}(W^{e_1}=-1)=0.5, \mathbb{P}(W^{e_1}=1)=0.5; \\
e_2: \mathbb{P}(Z^{e_2}_{spu} = 2)=1, \mathbb{P}(W^{e_2}=-1)=0.5, \mathbb{P}(W^{e_2}=1)=0.5; \\
e_3: \mathbb{P}(Z^{e_3}_{spu} = 1)=1, \mathbb{P}(W^{e_3}=-2)=0.5, \mathbb{P}(W^{e_3}=2)=0.5.
\end{gather*}
\end{Example}
Figure \ref{fig-example1} shows the support points of these features in the training environments.
Then, by applying any algorithm to solve the above example with $r^{th}=q$, we would get a predictor of $f^*=w^*\circ\Phi^*$. Consider the prediction made by this model as (we ignore the classifier bias for convenience)
\begin{equation}
f^*(X^e) = f^* (S(Z^e_{inv}, Z^e_{spu})) = \bm{1}(\Phi_{inv}^* Z^e_{inv} + \Phi_{spu}^* Z^e_{spu}).
\end{equation}
It is trivial to show that the $f^*$ of $\Phi_{inv}^*=0$ and $\Phi_{spu}^*=1$ is an invariant predictor across training environments with classification error $R^{e_1} = R^{e_2} = q$, and it achieves the least entropy of $h^e(\Phi^*)=0$ for each training environment $e$, and therefore, it is a solution of IB-IRM and IRM. However, the predictor of $f^*$ relies on spurious features and has the test error $R^{e_3} = 0.5$, thus fails to solve the OOD generalization.

\subsection{Failure under the Spurious correlation 3}
\begin{Example}[Counter-Example 2]\label{example-2}
Under Assumption \ref{assp-mode-3}, let $Z^e_{spu} \leftarrow W_1^e Y^e + W_0^e (1-Y^e)$ with $dim(Z_{inv})=dim(Z_{spu})=dim(W_0^e) = dim(W_1^e)=1$, $Z_{inv}^e$ be a discrete variable supported uniformly on six points $\{-4, -3, -2, 2, 3,4\}$ among all environments, and $w^*_{inv}=1$ be the generated classifier in Assumption \ref{assp-linearSEM}. We assume two training environments and a OOD environment being as:
% \begin{equation*}
% \mathcal{E}_{tr} = \{P(W^{e_1}=0.4)=P(W^{e_1}=-0.4)=0.5,\ P(W^{e_2}=0.3)=P(W^{e_2}=-0.3)=0.5\}
% \end{equation*}
\begin{gather*}
\mathcal{E}_{tr} = \{e_1, e_2\}; \quad \mathcal{E}_{ood} = \{e_3\}\\
e_1: \mathbb{P}(W^{e_1}_0=-1)=1, \mathbb{P}(W^{e_1}_1=1)=1; \\
e_2: \mathbb{P}(W^{e_2}_0=-0.5)=1,\mathbb{P}(W^{e_2}_1=0.5)=1; \\
e_3: \mathbb{P}(W^{e_3}_0=1)=1, \mathbb{P}(W^{e_3}_1=-1)=1;
\end{gather*}
\end{Example}
Figure \ref{fig-example2} shows the support points of these features in the training environments.
Then, by applying any algorithm to solve the above example with $r^{th}=q$, we would get a predictor of $f^*=w^*\circ\Phi^*$. Consider the prediction made by this model as (we ignore the classifier bias for convenience)
\begin{equation}
f^*(X^e) = f^* (S(Z^e_{inv}, Z^e_{spu})) = \bm{1}(\Phi_{inv}^* Z^e_{inv} + \Phi_{spu}^* Z^e_{spu}).
\end{equation}
It is trivial to show that the $f^*$ of $\Phi_{inv}^*=0$ and $\Phi_{spu}^*=1$ is an invariant predictor across training environments with classification error $R^{e_1} = R^{e_2} = 0$, and it achieves the least entropy of $h^e(\Phi^*)=1$ among all highly predictive predictors for each training environment $e$, and therefore, it is a solution of IB-IRM and IRM. However, the predictor of $f^*$ relies on spurious features and has the test error $R^{e_3} = 1$, thus fails to solve the OOD generalization.

\subsection{Understanding the failures}
From the illustrations of above simple examples, we can conclude that the failure of invariance constraint for removing the spurious features out is because the spurious features among all training environments are strictly linearly separable by their corresponding labels. This would make the predictor relying only on spurious features to achieve minimum training error and also be the invariant predictor across training environments. Since the label set is finite (with only two values in binary classification) in classification problems, such phenomenon may exist. We state such failure mode formally as below.
\begin{Theorem}\label{theorem-analysis2}
Given any $\mathcal{E}_{tr}\subset \mathcal{E}_{all}$ and $\mathcal{E}_{ood}\subseteq \mathcal{E}_{all}$ satisfying Assumptions \ref{assp-bound_inv_feature}, \ref{assp-sep_inv_feature}, and \ref{assp-inv_feature_overlap}, if two sets $\cup_{e\in\mathcal{E}_{tr}}\mathcal{Z}^e_{spu}(Y^e=1)$ and $\cup_{e\in \mathcal{E}_{tr}}\mathcal{Z}^e_{spu}(Y^e=0)$ are linearly separable and $H(Z^e_{inv})>H(Z^e_{spu})$ on each training environment $e$, then IB-IRM (and IRM, ERM, or IB-ERM) with any $r^{th}\in \mathbb{R}$ fails to solve the OOD generalization.
\end{Theorem}
The understanding of Theorem \ref{theorem-analysis2} is intuitive since when the spurious features in the training environments with respect to different labels are linearly separable, there is no algorithm that can distinguish spurious features from invariant features. Although the assumption of linear separation of the spurious features seems strong for this failure, it is easy to hold in high-dimensional space when $dim(Z_{spu})$ is large (common cases in practice such as image data). We have show one case in Appendix \ref{appendix-experiments} that if the number of environments $|\mathcal{E}_{tr}| < dim(Z_{spu}) / 2$ under the Assumption \ref{assp-mode-3}, the spurious features in the training environments are probably separable by their labels. This is because, in $o$-dimensional space, there is a high probability that $o$ randomly drawn distinct points are linearly separable for any two subsets.

\section{Counterfactual supervision-based information bottleneck}\label{section-method}
\begin{figure}
    \centering
    \includegraphics[width=1.0\textwidth]{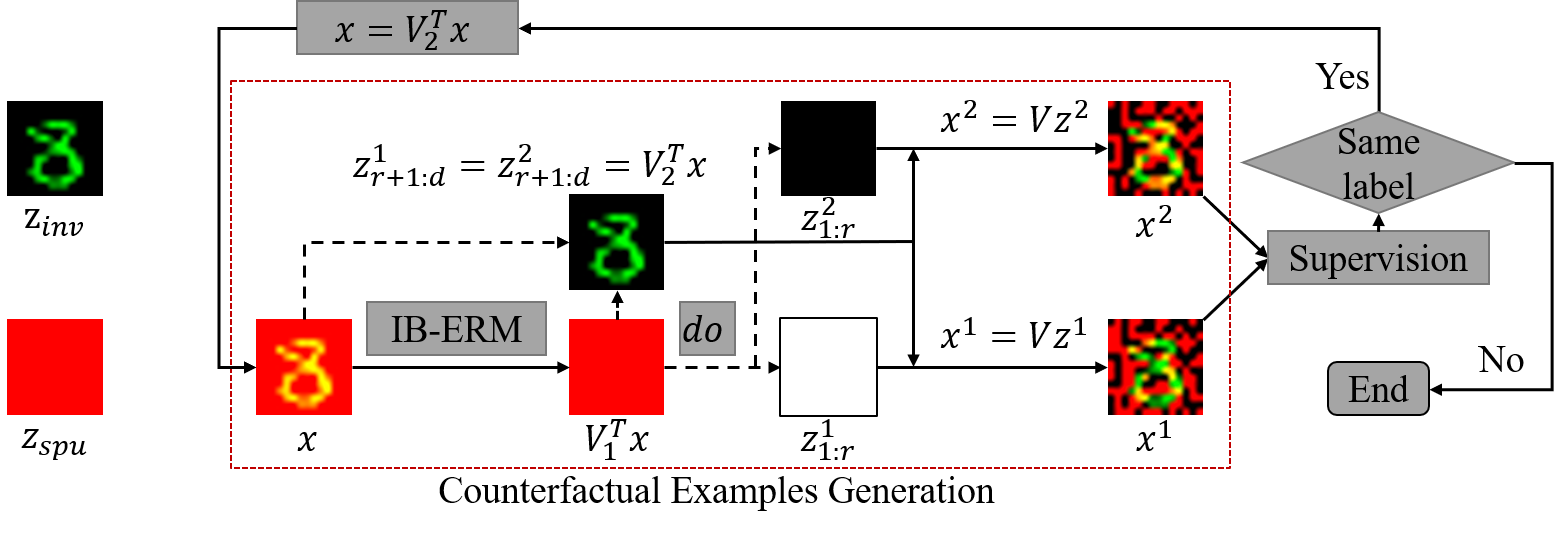}
    \caption{A simplified framework for the illustration of the proposed CSIB method.}
    \label{fig-framework}
\end{figure}
In the above analyses, we have shown two failure modes of IB-IRM and IRM for OOD generalization in the linear classification problem. The key reason for the failure is due to the learned features $\Phi(X)$ that rely on spurious features. To prevent such failure, we present counterfactual supervision-based information bottleneck (CSIB) learning algorithm for removing the spurious features progressively.

In general, IB-ERM method is applied to extract features from the begin of each iteration:
\begin{equation}\label{equ-IBERM}
 \quad \min_{w, \Phi} \sum_{e\in\mathcal{E}_{tr}} h^e(\Phi) \quad \text{s.t.} \ \frac{1}{|\mathcal{E}_{tr}|}\sum_{e\in \mathcal{E}_{tr}} R^e(w\circ\Phi) \leq r^{th}
\end{equation}
Due to the information bottleneck, only a part of information of the input $X$ are exploited in $\Phi(X)$. If the information of spurious features $Z_{spu}$ exists in the learned features $\Phi(X)$, the idea of CSIB is going to drop such information and meanwhile maintain the causal information (represented by invariant features $Z_{inv}$) as well. However, achieving such goal faces two challenges: (1) How to determine whether $\Phi(X)$ contains spurious information of $Z_{spu}$? and (2) How to remove the information of $Z_{spu}$?

Fortunately, due to the orthogonality in the linear space, it is possible to disentangle the features that are exploited by $\Phi(X)$ (denote as $X_1$) and the features that are not exploited by $\Phi(X)$ (denote as $X_2$) via Singular Value Decomposition (SVD). Base on that, we could construct a SEM $\mathcal{C}_{new}$ governing three variables of $X_1$, $X_2$, and $X$. Therefore, by doing counterfactual interventions on $X_1$ and $X_2$ in $\mathcal{C}_{new}$, we could solve the first challenge by requiring a single supervision on the counterfactual examples $X'$. For example, if we intervene on $X_1$ and find that the causal information remains in the resulting $X'$, then the extracted features $\Phi(X)$ are definitely the spurious features. To address the second challenge, we replace the input by $X_2$ by filtering out the information of $X_1$, and do the same learning procedure from the beginning.

The learning algorithm of CSIB is illustrated in Algorithm \ref{alg-CSIB}, and Figure \ref{fig-framework} shows the framework of CSIB. We show in Theorem \ref{theorem-CSIB} that CSIB is theoretically guaranteed to succeed to solve OOD generalization.

\begin{algorithm}
    \caption{Counterfactual Supervision-based Information Bottleneck (CSIB)}\label{alg-CSIB}
    % \hspace*{\algorithmicindent}
    \textbf{Input:} $\mathbb{P}(X^e, Y^e)$, $e\in \mathcal{E}_{tr}$, $r^{th} > 0$, $c\geq dim(Z_{inv})$, $M \gg 0$, and $(x, y)$ is an example randomly drawn from $\mathbb{P}(X^e, Y^e)$. \\
    % \hspace*{\algorithmicindent}
    \textbf{Output:} classifier $w\in\mathbb{R}^{c+1}$, feature extractor $\Phi=\mathbb{R}^{c\times d}$. \\
    \textbf{Begin:}
    \begin{algorithmic}[1]
    % \Require $M \gg 0$
    \State $Lv\gets \text{[]}$; $ Lr\gets \text{[]}$; $\Phi' \gets \mathbb{I}^{d\times d}$
    \State $d' \gets dim(X^e)$
    \State Apply IB-ERM method (Equation \ref{equ-IBERM}) to $\mathbb{P}(X^e, Y^e)$ and get $w^*\in\mathbb{R}^{c+1}$ and $\Phi^*\in\mathbb{R}^{c\times d'}$
    \State Apply SVD to $\Phi^*$ as $\Phi^* = U\Lambda V^T = [U_1, U_2][\Lambda_1, \bm{0}; \bm{0}, \bm{0}] [V_1^T; V_2^T]$
    \State $r \gets rank(\Phi^*)$
    \State $z^1_{1:r} \gets [-M, .., -M]$; $z^1_{r+1:d'} \gets V_2^T \Phi' x$
    \State $z^2_{1:r} \gets [M, .., M]$; $z^2_{r+1:d'} \gets V_2^T \Phi' x$
    \State $x^1 \gets V z^1$; $x^2 \gets V z^2$
    \If{$Lv$ is not empty}
        \State $z_{old}\gets \text{[]}$; $i\gets 0$; $x'\gets x$
        \While{$i<len(Lv)$}
        \State $z \gets Lv[i] x'$
        \State $z_{old}$.append($z$)
        \State $x' \gets z_{Lr[i]:}$
        \State $i\gets i+1$
        \EndWhile
        \State $i\gets 0$
        \While{$i<len(Lv)$}
        \State $j\gets len(Lv) - i$
        \State $z^1 \gets z_{old}[j]$; $z^2 \gets z_{old}[j]$
        \State $z^1_{Lr[j]:} \gets x^1$; $z^2_{Lr[j]:} \gets x^2$
        \State $x^1 \gets Lv[j]^T z^1$; $x^2 \gets Lv[j]^T z^2$
        \State $i\gets i+1$
        \EndWhile
    \EndIf
    \If{label($x^1$) $=$ label($x^2$)}
        \State $Lr$.append($r$); $Lv$.append($V^T$)
        \State $X^e\gets V_2^T X^e$; $\Phi' \gets V_2^T \Phi'$
        \State Goto Step 2
    \EndIf
    \State $w\gets w^*$; $\Phi \gets \Phi^*$
    \end{algorithmic}
    \textbf{End}
\end{algorithm}

\begin{Theorem}[Guarantee of CSIB]\label{theorem-CSIB}
Given any $\mathcal{E}_{tr}\subset \mathcal{E}_{all}$ and $\mathcal{E}_{ood}\subseteq \mathcal{E}_{all}$ satisfying Assumptions \ref{assp-bound_inv_feature}, \ref{assp-sep_inv_feature}, and \ref{assp-new-assp}, then for every Spurious correlations of Assumptions \ref{assp-mode-1}, \ref{assp-mode-2}\footnote{In this correlation mode, assume the spurious features are linearly separable in the training environments}, and \ref{assp-mode-3}, CSIB algorithm with $r^{th} = q$ succeeds to solve the OOD generalization.
\end{Theorem}

% \textbf{Significance of Theorem \ref{theorem-CSIB}.}
\begin{remark}
CSIB succeeds to solve OOD generalization without assuming the support overlap to invariant features and could apply to multiple spurious modes where IB-IRM (as well as ERM, IRM, and IB-ERM) may fail. By introducing counterfactual inference and further supervision (usually done by human) with several steps, CSIB works even when accessing data from a single environment, which is significant especially in the cases where multiple environments data are not available.
\end{remark}

\section{Experiments}\label{section-experiments}
\subsection{Toy experiments on synthetic datasets}
We begin perform experiments on three synthetic datasets from different spurious correlations modes to verify our method -- counterfactual supervision-based information bottleneck (CSIB) -- and compare it to ERM, IB-ERM, IRM, and IB-IRM.
We follow the same protocol for tuning hyperparameters from \cite{arjovsky2019invariant, aubin2021linear, ahuja2021invariance} and report the classification error for all experiments. In the following, we first briefly describe the designed datasets and then report the main results. More experimental details can be found in Appendix.

\begin{table}
\caption{Summary of three synthetic datasets. Note that for linearly separable features, their margin levels significantly influence the final learning classifier due to the implicit bias of the gradient descent \cite{soudry2018implicit}. Such bias would push the standard learning (like cross-entropy loss) focusing more on the large-margin features. The margin with respect to a dataset (or features) $\mathcal{Z}$ (each instance has a label 0 or 1) is the minimum distance between a point in $\mathcal{Z}$ and the max-margin hyperplane, which separates $\mathcal{Z}$ by their labels.}\label{table-summary-of-datasets}
\centering
\adjustbox{max width=0.9\textwidth}{%
% \begin{adjustwidth}{-\extralength}{0cm}
% \newcolumntype{C}{>{\centering\arraybackslash}X}
\begin{tabularx}{\textwidth}{ccccc}
\toprule
    Datasets   &  Margin relationship & Entropy relationship & $\text{Dim}_{inv}$  & $\text{Dim}_{spu}$   \\
\midrule
    Example 1/1S &  $\text{Margin}_{inv} \ll \text{Margin}_{spu}$  & $\text{Entropy}_{inv} < \text{Entropy}_{spu}$ & 5 & 5  \\
    Example 2/2S &  $\text{Margin}_{inv} \approx \text{Margin}_{spu}$  & $\text{Entropy}_{inv} > \text{Entropy}_{spu}$ & 5 & 5 \\
    Example 3/3S &  $\text{Margin}_{inv} \gg \text{Margin}_{spu}$  & $\text{Entropy}_{inv} > \text{Entropy}_{spu}$ & 5 & 5 \\
\bottomrule
\end{tabularx}
% \end{adjustwidth}
}
\end{table}

\subsubsection{Datasets}
\textbf{Example 1/1S.} The example is a modified one from the linear unit tests introduced in \cite{aubin2021linear}, which generalizes the cow/camel classification task with relevant backgrounds.
\begin{align*}
& \theta_{cow} = \bm{1}_m, \quad \theta_{camel} = -\theta_{cow}, \quad \nu_{animal} = 10^{-2} \\
& \theta_{grass} = \bm{1}_o,\quad  \theta_{sand} = -\theta_{grass}, \quad \nu_{background} = 1.
\end{align*}
The dataset $D_e$ of each environment $e\in\mathcal{E}_{tr}$ is sampled from the following distribution
\begin{gather*}
%\begin{align*}
 U^e \sim \text{Categorical}(p^es^e, (1-p^e)s^e, p^e(1-s^e), (1-p^e)(1-s^e)), \\
 Z^e_{inv} \sim \left\{
\begin{aligned}
& (\mathcal{N}_m(0, 0.1) + \theta_{cow})\nu_{animal} &\quad \text{if}\ U^e\in\{1,2\}, \\
& (\mathcal{N}_m(0, 0.1) + \theta_{camel})\nu_{animal}& \quad \text{if}\ U^e\in\{3,4\},
\end{aligned}
\right. \\
 Z^e_{spu} \sim \left\{
\begin{aligned}
& (\mathcal{N}_o(0, 0.1) + \theta_{grass})\nu_{background} &\quad \text{if}\ U^e\in\{1,4\}, \\
& (\mathcal{N}_o(0, 0.1) + \theta_{sand})\nu_{background}& \quad \text{if}\ U^e\in\{2,3\},
\end{aligned}
\right. \\
 Z^e \leftarrow (Z^e_{inv}, Z^e_{spu}), \quad X^e \leftarrow S(Z^e), \quad N \sim Bernoulli(q), \ q< 0.5, \quad Y^e \leftarrow \bm{1}(\bm{1}^T_m Z^e_{inv}) \oplus N
%\end{align*}
\end{gather*}
We set $s^{e_0}=0.5, s^{e_1}=0.7, s^{e_2}=0.3$ for the first three environments, and $s^{e_j}\sim \text{Uniform}(0.3,0.7)$ for $j>3$. The scrambling matrix $S$ is an identical matrix in Example 1 and a random unitary matrix in Example 1S. Here, we set $p^e=1$ and $q=0$ for all environments to make the spurious features and the invariant features both linearly separable to confuse each other. For the experiments on different values of $q$ and $p^e$ are presented in Appendix, where we have found very interesting observations related to the inductive bias of neural networks.

\textbf{Example 2/2S.} This example is extended from the Example \ref{example-1} to show one of the failure modes of IB-IRM (as well as ERM, IRM, and IB-ERM) and how our method can be improved by intervention (counterfactual supervision). Given $w^e \in \mathbb{R}$, each instance in the environment data $D^e$ is sampled by
\begin{gather*}
\theta_{spu} = 5\cdot\bm{1}_o, \quad \theta_{w} = w^e\cdot\bm{1}_m, \quad \nu_{spu} = 10^{-2}, \quad \nu_{w} = 1, \quad p, q \sim \text{Bernoulli(0.5)}, \\ \quad Z^e_{spu} = \mathcal{N}_o(0,1)\nu_{spu} + (2p-1)\cdot\theta_{spu}, \quad W^e = \mathcal{N}_m(0,1)\nu_{w} + (2q-1)\cdot\theta_{w} \\
Z^e_{inv} = AZ^e_{spu} + W^e, \quad Z^e \leftarrow (Z^e_{inv}, Z^e_{spu}), \quad X^e \leftarrow S(Z^e),\quad Y^e = \bm{1}(\bm{1}^T_m Z^e_{inv}),
\end{gather*}
where we set $m=o=5$ and $A\in\mathbb{R}^{m\times o}$ be the identical matrix in our experiments. We set $w^{e_0}=3$, $w^{e_1}=2$, $w^{e_2}=1$, and $w^{e_j} = \text{Uniform}(0,3)$ if $j>3$ for different training environments. This example shows clear smaller entropy of spurious features than that of invariant features, which is opposite to the Example 1/1S.

\textbf{Example 3/3S.} This example extends from the Example \ref{example-2} and similar to the construction of Example 2/2S. Let $w^e\sim \text{Uniform}(0, 1)$ for different training environments. Each instance in the environments $e$ is sampled by
\begin{gather*}
\theta_{inv} = \cdot10\cdot\bm{1}_m, \quad \nu_{inv} = 10, \quad \nu_{spu}=1, \quad p,q \sim \text{Bernoulli(0.5)}, \\
Z^e_{inv} = \mathcal{N}_m(0,1)\nu_{inv} + (2p-1)\cdot\theta_{inv}, \quad Y^e = \bm{1}(\bm{1}^T_m Z^e_{inv}), \\
Z^e_{spu} = 2(Y^e-1)\cdot\nu_{spu} + (2q-1)\cdot w^e\cdot\bm{1}_o, \quad Z^e \leftarrow (Z^e_{inv}, Z^e_{spu}), \quad X^e \leftarrow S(Z^e),
\end{gather*}
where we set $m=o=5$ in our experiments. The spurious features have smaller entropy than the invariant features in this example, which is similar to Example 2/2S, but the invariant features significantly enjoy much larger margin than the spurious features, which is very different from the above two examples. We make a summary to the properties of these three datasets in Table \ref{table-summary-of-datasets} for a general view.

\begin{table}
\caption{Main results. \#Envs means the number of training environments, and (min) reports the minimal test classification error across different running seeds.}
\begin{center}
\adjustbox{max width=\textwidth}{%
% \begin{adjustwidth}{-\extralength}{0cm}
\begin{tabular}{lcccccc}
\toprule
            & \#Envs    & ERM (min)       & IRM (min)              & IB-ERM (min)          & IB-IRM (min)          & CSIB (min)                 \\
\midrule
Example 1   & 1   & 0.50 $\pm$ 0.01 (0.49) & 0.50 $\pm$ 0.01 (0.49) & \textbf{0.23 $\pm$ 0.02 (0.22)} & 0.31 $\pm$ 0.10 (0.25) & \textbf{0.23 $\pm$ 0.02 (0.22)}  \\
Example 1S  & 1   & 0.50 $\pm$ 0.00 (0.49) & 0.50 $\pm$ 0.00 (0.50) & \textbf{0.46 $\pm$ 0.04 (0.39)} & 0.30 $\pm$ 0.10 (0.25) & \textbf{0.46 $\pm$ 0.04 (0.39)}  \\
Example 2   & 1   & 0.40 $\pm$ 0.20 (0.00) & 0.50 $\pm$ 0.00 (0.49) & 0.50 $\pm$ 0.00 (0.49) & 0.46 $\pm$ 0.02 (0.45) & \textbf{0.00 $\pm$ 0.00 (0.00)}  \\
Example 2S  & 1   & 0.50 $\pm$ 0.00 (0.50) & 0.31 $\pm$ 0.23 (0.00) & 0.50 $\pm$ 0.00 (0.50) & 0.45 $\pm$ 0.01 (0.43) & \textbf{0.10 $\pm$ 0.20 (0.00)}  \\
Example 3   & 1   & 0.16 $\pm$ 0.06 (0.09) & 0.18 $\pm$ 0.03 (0.14) & 0.50 $\pm$ 0.01 (0.49) & 0.40 $\pm$ 0.20 (0.01) & \textbf{0.11 $\pm$ 0.20 (0.00)}  \\
Example 3S  & 1   & 0.17 $\pm$ 0.07 (0.10) & \textbf{0.09 $\pm$ 0.02 (0.07)} & 0.50 $\pm$ 0.00 (0.50) & 0.50 $\pm$ 0.00 (0.50) & 0.21 $\pm$ 0.24 (0.00)  \\ \hline
Example 1   & 3   & 0.45 $\pm$ 0.01 (0.45) & 0.45 $\pm$ 0.01 (0.45) & \textbf{0.22 $\pm$ 0.01 (0.21)} & 0.23 $\pm$ 0.13 (0.02) & \textbf{0.22 $\pm$ 0.01 (0.21)}  \\
Example 1S  & 3   & 0.45 $\pm$ 0.00 (0.45) & 0.45 $\pm$ 0.00 (0.45) & 0.41 $\pm$ 0.04 (0.34) & \textbf{0.27 $\pm$ 0.11 (0.11)} & 0.41 $\pm$ 0.04 (0.34)  \\
Example 2   & 3   & 0.40 $\pm$ 0.20 (0.00) & 0.50 $\pm$ 0.00 (0.50) & 0.50 $\pm$ 0.00 (0.50) & 0.33 $\pm$ 0.04 (0.25) & \textbf{0.00 $\pm$ 0.00 (0.00)}  \\
Example 2S  & 3   & 0.50 $\pm$ 0.00 (0.50) & 0.37 $\pm$ 0.15 (0.15) & 0.50 $\pm$ 0.00 (0.50) & 0.34 $\pm$ 0.01 (0.33) & \textbf{0.10 $\pm$ 0.20 (0.00)}  \\
Example 3   & 3   & 0.18 $\pm$ 0.04 (0.15) & 0.21 $\pm$ 0.02 (0.20) & 0.50 $\pm$ 0.01 (0.49) & 0.50 $\pm$ 0.01 (0.49) & \textbf{0.11 $\pm$ 0.20 (0.00)}  \\
Example 3S  & 3   & 0.18 $\pm$ 0.04 (0.15) & 0.08 $\pm$ 0.03 (0.03) & 0.50 $\pm$ 0.00 (0.50) & 0.43 $\pm$ 0.09 (0.31) & \textbf{0.01 $\pm$ 0.00 (0.00)}  \\ \hline
Example 1   & 6   & 0.46 $\pm$ 0.01 (0.44) & 0.46 $\pm$ 0.09 (0.41) & \textbf{0.22 $\pm$ 0.01 (0.20)} & 0.37 $\pm$ 0.14 (0.17) & \textbf{0.22 $\pm$ 0.01 (0.20)}  \\
Example 1S  & 6   & 0.46 $\pm$ 0.02 (0.44) & 0.46 $\pm$ 0.02 (0.44) & 0.35 $\pm$ 0.10 (0.23) & 0.42 $\pm$ 0.12 (0.28) & \textbf{0.35 $\pm$ 0.10 (0.23)}  \\
Example 2   & 6   & 0.49 $\pm$ 0.01 (0.48) & 0.50 $\pm$ 0.01 (0.48) & 0.50 $\pm$ 0.00 (0.50) & 0.30 $\pm$ 0.01 (0.28) & \textbf{0.00 $\pm$ 0.00 (0.00)}  \\
Example 2S  & 6   & 0.50 $\pm$ 0.00 (0.50) & 0.35 $\pm$ 0.12 (0.25) & 0.50 $\pm$ 0.00 (0.50) & 0.30 $\pm$ 0.01 (0.29) & \textbf{0.20 $\pm$ 0.24 (0.00)}  \\
Example 3   & 6   & 0.18 $\pm$ 0.04 (0.15) & 0.20 $\pm$ 0.01 (0.19) & 0.50 $\pm$ 0.00 (0.49) & 0.37 $\pm$ 0.16 (0.16) & \textbf{0.01 $\pm$ 0.01 (0.00)}  \\
Example 3S  & 6   & 0.18 $\pm$ 0.04 (0.14) & \textbf{0.05 $\pm$ 0.04 (0.01)} & 0.50 $\pm$ 0.00 (0.50) & 0.50 $\pm$ 0.00 (0.50) & 0.11 $\pm$ 0.20 (0.00)  \\
\bottomrule
\end{tabular}
% \end{adjustwidth}
}
\end{center}
\label{table-main_results}
\end{table}

\subsubsection{Summary of results}\label{section-summaryofresults}
Table \ref{table-main_results} shows the classification errors of different methods when training data comes from single, three, and six environments. We can see that ERM and IRM fail to recognize the invariant features in the experiment of Example 1/1S, where invariant features have smaller margin than spurious features do, while information bottleneck-based methods (IB-ERM, IB-IRM, and CSIB) show improved results due to the smaller entropy of the invariant features. Our method CSIB shows consistent results with IB-IRM in Example 1/1S when invariant features are extracted in the first run, which verifies the effectiveness of using information bottleneck for OOD generalization. In another dataset of Example 2/2S, where the invariant features have larger entropy than spurious features do, we can see that only CSIB can remove the spurious features out among all comparing methods, although information bottleneck-based method IB-ERM would degrade the performance of ERM by focusing more on the spurious features. In the third experiments of Example 3/3S, we can see that although ERM shows not-bad results due to the significantly larger margin of invariant features, our method CSIB still shows improvements by removing out more spurious features. Notably, comparing to the IB-ERM and IB-IRM when only spurious features are extracted (Example 2/2S, Example 3/3S), our method CSIB could effectively remove them by counterfactual supervision and then refocus on the invariant features. Note that the reason of non-zero average error and the fluctuant results of CSIB in some experiments is because the entropy minimization in the training process is less accurate, where entropy is substituted by variance for the ease of the optimization. Nevertheless, there always exists a case where the entropy is indeed truly minimized and the error reaches zero (see (min) in the table) in Example 2/2S and Example 3/3S. In summary, CSIB improves others consistently from different spurious correlations modes and are especially more effective than IB-ERM and IB-IRM when the spurious features enjoy much smaller entropy than the invariant features do.

\subsection{Experiments on color mnist dataset}
In this experiment, we set up a binary classification task for digit recognition -- identify whether the digit is less than 5 or more than 5. We use real-world dataset, the MNIST database of handwritten digits \footnote{\url{http://yann.lecun.com/exdb/mnist/}}, for the construction. Following our learning setting, we use color information as the spurious features that correlates strongly with the class label. By construction, the label is strongly correlated with the color than with the digit in the training environments but this correlation is broken in the test environment. Specifically, the designed three environments (two training environments and one test environment containing 10000 points each) of the color mnist are as follows: first, define a preliminary binary label $\hat{y}$ to the image base on the digit: $\hat{y}=0$ for digits 0-4 and $\hat{y}=1$ for 5-9. Second, obtain the final label $y$ by flipping $\hat{y}$ with probability 0.25. Then, we flip the final labels to obtain the color id, where the flipping probabilities with respect to two training environments and one test environment are 0.2 and 0.1, and 0.9. For better understanding, we randomly draw 20 examples for each label from each environment and visualize them in Figure \ref{fig-colormnist}.
\begin{figure}
    \centering
    \includegraphics[width=0.8\textwidth]{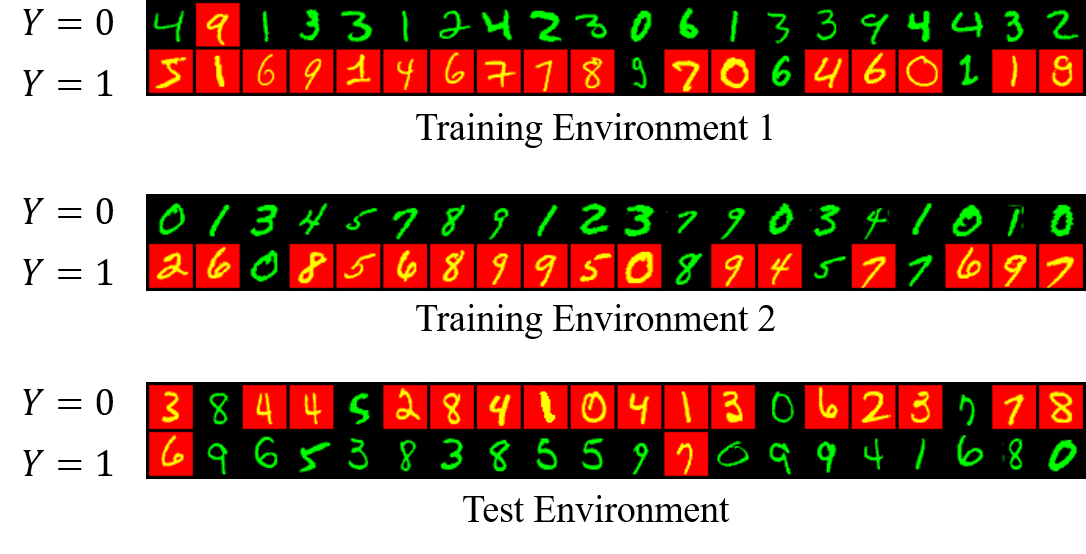}
    \caption{Visualization of the color mnist dataset.}
    \label{fig-colormnist}
\end{figure}

\begin{table}
    \centering
    \caption{Classification accuracy (\%) on color mnist dataset. "Oracle" in the table means that the training and test data are in the same environment.}
    \label{table-colormnist}
    \adjustbox{max width=\textwidth}{%
    \begin{tabular}{cccccc|c}
    \toprule
    Methods & ERM & IRM & IB-ERM & IB-IRM & CSIB & Oracle \\
    \midrule
    % Acc &  10.04 $\pm$ 0.22 & 22.76 $\pm$ 4.59 & 10.04 $\pm$ 0.22 & 38.75 $\pm$ 12.71 & \textbf{60.22 $\pm$ 0.74} & 83.83 $\pm$ 1.71 \\
    Accuracy &  9.94 $\pm$ 0.28 & 20.39 $\pm$ 2.76 & 9.94 $\pm$ 0.28 & 43.84 $\pm$ 12.48 & \textbf{60.03 $\pm$ 1.28} & 84.72 $\pm$ 0.65 \\
    \bottomrule
    \end{tabular}
    }
\end{table}
The classification results on color mnist dataset are shown in Table \ref{table-colormnist}. From the results, we can see that both ERM and IB-ERM methods are almost surely using the color features to achieve the task. Although IRM and IB-IRM methods have shown some improvements over ERM, only our method can perform better than a random prediction, which demonstrates the effectiveness of CSIB.

\section{Related works}\label{section-related-work}
We divide the works related to OOD generalization into two categories: theory and methods, though some of them belong to both.
\subsection{Theory of OOD generalization}
Based on different definitions to the distributional changes, we review the corresponding theory by the following three categories.

\textbf{Based on causality.} Due to the close connection between the distributional changes and the interventions discussed in the theory of causality \cite{pearl2009causality, peters2017elements}, the problem of OOD generalization is usually built in the framework of causal learning. The theory states that a response $Y$ is directly caused only by its parents variables $X_{Pa(Y)}$ and all interventions other that those on $Y$ do not change the conditional distribution of $\mathbb{P}(Y|X_{Pa(Y)})$. Such theory inspires a popular learning principle -- invariance principle -- that aims to discover a set of variables such that they remain invariance to the response $Y$ in all observed environments \cite{peters2016causal, heinze2018invariant, rojas2018invariant}. Invariant risk minimization (IRM) \cite{arjovsky2019invariant} is then proposed to learn a feature extractor $\Phi$ in an end-to-end way such that the optimal classifier based on the extracted features $\Phi(X)$ remains unchange in each environment. The theory in \cite{arjovsky2019invariant} shows the guarantee of IRM for OOD generalization under some general assumptions, but only focuses on the linear regression tasks. Different from the failure analyses of IRM for the classification tasks in \cite{rosenfeld2020risks, kamath2021does}, where the response Y is the cause of the spurious feature, Ahuja et al., \cite{ahuja2021invariance} analyse another scenario when the invariant feature is the cause of the spurious feature, and show that in this case, linear classification is more difficult than linear regression, where the invariance principle itself is insufficient to ensure the success of OOD generalization, and they also claim that the assumption of support overlap of invariant features is necessarily needed. They then propose a learning principle of information bottleneck-based invariant risk minimization (IB-IRM) for linear classification, which shows how to address the failures of IRM by adding information bottleneck \cite{tishby1999information} into the learning. In this work, we closely investigate the conditions identified in \cite{ahuja2021invariance} and first show that support overlap of invariant features is not necessarily needed for the success of OOD generalization. We further show several failure cases of IB-IRM and propose improved results to it.

Recently, some of works are proposed to tackle the challenge of OOD generalization in the non-linear regime \cite{lu2022invariant,liu2021learning}. Commonly, both of them use variational auto encoder (VAE)-based models \cite{kingma2013auto, rezende2014stochastic} to identify the latent variables from observations in the first stage. Then, these inferring latent variables are separated to two distinct parts of invariant (causal) and spurious (non-causal) features based on different assumptions to them. Specifically, Lu et al. \cite{lu2021nonlinear,lu2022invariant} assume that the latent variables conditioned on some accessible side information such as the environment index or class label are followed the exponential family distributions, and Liu et al. \cite{liu2021learning} directly disentangle the latent variables to two different parts during the inferring stage and assumes that the marginal distributions of them are independent to each other. These assumptions, however, are rather strong in general. Nevertheless, these solutions aim to capture the latent variables such that the response given these variables is invariant for different environments, which could still fail because the invariance principle itself is insufficient for OOD generalization in the classification tasks, as shown in \cite{ahuja2021invariance}.
In this work, we focus on the linear classification only and show a new theory of a new method that well addresses several OOD generalization failures in the linear settings. Our method could extend to the nonlinear regime by combing with the disentangled representation learning \cite{bengio2013representation} or causal representation learning \cite{scholkopf2021toward}. Specifically, once the latent representations are well disentangled, i.e., the latent features are represented by a linear transform of the causal features and spurious features, we then could apply our method to filter out the spurious features in the latent space such that only causal features remain.

\textbf{Based on robustness.} Different from those based on the causality, where different distributions are generated by intervention on a same SEM and the goal is to discover causal features, the robustness based methods aim to protect the model against the potential distributional shifts within the uncertainty set, which is usually constrained by f-divergence \cite{namkoong2016stochastic} or Wasserstein distance \cite{sinha2017certifying}. This series of works are theoretically addressed by distributionally robust optimization (DRO) under a minimax framework \cite{lee2018minimax,duchi2021learning}. Recently, some of works tend to discover the connections between causality and robustness \cite{buhlmann2020invariance}. Although these works show less relevance to us, it is possible that a well-defined measure of distribution divergence could help to effectively extract causal features under the robustness framework. This would be an interesting avenue for future research.

\textbf{Others.} Some other works assume that the distributions (domains) are generated from a hyper-distribution and aims to minimize the average risk estimation error bound \cite{blanchard2011generalizing,muandet2013domain,deshmukh2019generalization}. These works are often built based on the generalization theory under the independent and identically distributed (IID) assumption. And in \cite{ye2021towards}, it does not make any assumption to the distributional changes, and only studies the learnability of OOD generalization in a general way. All of these theories do not cover the OOD generalization problem under a single training environment or domain.

\subsection{Methods of OOD generalization}
\textbf{Based on invariance principle.} Inspired from the invariance principle \cite{peters2016causal,heinze2018invariant}, many methods are proposed by designing various loss to extract features to better satisfy the principle itself. IRMv1 \cite{arjovsky2019invariant} is the first objective to address this in an end-to-end way by adding a gradient penalty to the classifier. Following this work, Krueger et al. \cite{krueger2021out} suggest penalizing the variance of the risks, while Xie et al. \cite{xie2020risk} give the same objective but taking the square root of the variance. And many other alternatives could also be found \cite{jin2020domain,mahajan2021domain,bellot2020generalization}. It is clear that all of these methods aims to find an invariant predictor. Recently, Ahuja et al. \cite{ahuja2021invariance} find that for classification problem, finding the invariant predictor is not enough to extract causal features since the features could include the spurious information to make the predictor invariant across training environments, and they propose IB-IRM to address such failure. Similar ideas to IB-IRM could also be found in the work \cite{li2021invariant, alesiani2022gated}, where different loss functions are proposed to achieve the same purpose. Specifically, Alesiani et al. \cite{alesiani2022gated} also use information bottleneck (IB) for the help to drop spurious correlations, but their analyses only focus on the scenario when spurious features are independent to the causal features, which could be considered as a special case of ours. More recently, Wang et al. \cite{InvariantSubDG} propose the similar ideas to ours but only tackle the situation when the invariant features have the same distribution among all environments. In this work, we further show that IB-IRM could still fail in two cases due to the model may only rely on spurious features to meet the task of interest. We then propose counterfactual supervision-based information bottleneck (CSIB) method to address such failures and show improving results to the prior works.

\textbf{Based on distribution matching.} It is worth to note that there exist many works focused on learning domain invariant features representations \cite{ganin2015unsupervised,li2018deep,zhao2020domain}. Most of these works are inspired by the seminal theory of domain adaptation \cite{ben2006analysis, ben2010theory}. The goal of these methods is to learn a feature extractor $\Phi$ such that the marginal distribution of $\mathbb{P}(\Phi(X))$ or the conditional distribution of $\mathbb{P}(\Phi(X)|Y)$ is invariant across different domains. This is different from the invariance principle, where the goal is to make $\mathbb{P}(Y|\Phi(X))$ (or $\mathbb{E}(Y|\Phi(X))$) invariant. We refer readers to the papers of \cite{arjovsky2019invariant,zhao2019learning} for better understanding the details of why these distribution matching based methods often fail to address OOD generalization.

\textbf{Others.} Other related methods are various, including by using data augmentation in both image level \cite{xu2021fourier}
or feature level \cite{zhou2020domain}, by removing spurious correlations through stable learning \cite{zhang2021deep}, and by utilizing the inductive bias of neural network \cite{geirhos2018imagenet,wang2019learning} etc. Most of these methods are empirically inspired from the experiments and are verified to some specific datasets. Recently, empirical studies in \cite{gulrajani2020search,wiles2022a} notice that the real effects of many OOD generalization (domain generalization) methods are weak, which indicates that the benchmark-based evaluation criterions may be inadequate to validate the OOD generalization algorithms.

\section{Conclusion, limitations and future work}\label{section-conclusion-and-limitation}
In this paper, we focus on the OOD generalization problem of linear classification. We first revisit the fundamental assumptions and results of prior works and show that the condition of invariant features support overlap is not necessarily needed for the success of OOD generalization and thus propose a weaker counterpart. Then, we show two failure cases of IB-IRM (as well as ERM, IB-ERM, and IRM) and illustrate its intrinsic causes by theoretical analysis. Motivating by that, we further propose a new method -- counterfactual supervision-based information bottleneck (CSIB) and theoretically prove its effectiveness under some weaker assumptions. CSIB works even when accessing data from a single environment, and can easily extend to the multi-class problems. Finally, we design several synthetic datasets by our motivating examples for the experimental verification. Empirical observations among all comparing methods illustrate the effectiveness of CSIB.

Since we only take the linear problem into interest, including linear representation and linear classifier, any non-linear case of that would not be guaranteed by our theoretical results and thus CSIB may fail. Therefore, the same as prior works (IRM \cite{arjovsky2019invariant} and IB-IRM \cite{ahuja2021invariance}), non-linear challenge is still an unsolved problem \cite{rosenfeld2020risks, kamath2021does}. We believe this is of great value for investigating in future work since widely used data in the wild are non-linearly generated. Another fruitful direction is to design a powerful algorithm for entropy minimization during the learning process of CSIB. Currently, we use the variance of features to replace the entropy of the features during the optimization. However, variance and entropy are essentially different but a truly effective entropy minimization is the key to the success of CSIB. Another limitation of our method is that we have to require a further supervision to the counterfactual examples during the learning process, although it only takes one time for a single step.


\bibliographystyle{plain}
\bibliography{bing}

\begin{thebibliography}{10}

\bibitem{ahuja2021invariance}
Kartik Ahuja, Ethan Caballero, Dinghuai Zhang, Jean-Christophe Gagnon-Audet,
  Yoshua Bengio, Ioannis Mitliagkas, and Irina Rish.
\newblock Invariance principle meets information bottleneck for
  out-of-distribution generalization.
\newblock In {\em Neural Information Processing Systems}, volume~34, 2021.

\bibitem{ahuja2020invariant}
Kartik Ahuja, Karthikeyan Shanmugam, Kush Varshney, and Amit Dhurandhar.
\newblock Invariant risk minimization games.
\newblock In {\em International Conference on Machine Learning}, pages
  145--155. PMLR, 2020.

\bibitem{alesiani2022gated}
Francesco Alesiani, Shujian Yu, and Xi~Yu.
\newblock Gated information bottleneck for generalization in sequential
  environments.
\newblock {\em Knowledge and Information Systems}, pages 1--23, 2022.

\bibitem{arjovsky2019invariant}
Martin Arjovsky, L{\'e}on Bottou, Ishaan Gulrajani, and David Lopez-Paz.
\newblock Invariant risk minimization.
\newblock {\em arXiv preprint arXiv:1907.02893}, 2019.

\bibitem{aubin2021linear}
Benjamin Aubin, Agnieszka S{\l}owik, Martin Arjovsky, Leon Bottou, and David
  Lopez-Paz.
\newblock Linear unit-tests for invariance discovery.
\newblock {\em arXiv preprint arXiv:2102.10867}, 2021.

\bibitem{beery2018recognition}
Sara Beery, Grant Van~Horn, and Pietro Perona.
\newblock Recognition in terra incognita.
\newblock In {\em European Conference on Computer Vision}, pages 456--473,
  2018.

\bibitem{bellot2020generalization}
Alexis Bellot and Mihaela van~der Schaar.
\newblock Generalization and invariances in the presence of unobserved
  confounding.
\newblock {\em arXiv preprint arXiv:2007.10653}, 11, 2020.

\bibitem{ben2010theory}
Shai Ben-David, John Blitzer, Koby Crammer, Alex Kulesza, Fernando Pereira, and
  Jennifer~Wortman Vaughan.
\newblock A theory of learning from different domains.
\newblock {\em Machine learning}, 79(1):151--175, 2010.

\bibitem{ben2006analysis}
Shai Ben-David, John Blitzer, Koby Crammer, and Fernando Pereira.
\newblock Analysis of representations for domain adaptation.
\newblock In {\em Neural Information Processing Systems}, volume~19, 2006.

\bibitem{bengio2013representation}
Yoshua Bengio, Aaron Courville, and Pascal Vincent.
\newblock Representation learning: A review and new perspectives.
\newblock {\em IEEE transactions on pattern analysis and machine intelligence},
  35(8):1798--1828, 2013.

\bibitem{blanchard2011generalizing}
Gilles Blanchard, Gyemin Lee, and Clayton Scott.
\newblock Generalizing from several related classification tasks to a new
  unlabeled sample.
\newblock In {\em Neural Information Processing Systems}, volume~24, 2011.

\bibitem{buhlmann2020invariance}
Peter B{\"u}hlmann.
\newblock Invariance, causality and robustness.
\newblock {\em Statistical Science}, 35(3):404--426, 2020.

\bibitem{deshmukh2019generalization}
Aniket~Anand Deshmukh, Yunwen Lei, Srinagesh Sharma, Urun Dogan, James~W
  Cutler, and Clayton Scott.
\newblock A generalization error bound for multi-class domain generalization.
\newblock {\em arXiv preprint arXiv:1905.10392}, 2019.

\bibitem{duchi2021learning}
John~C Duchi and Hongseok Namkoong.
\newblock Learning models with uniform performance via distributionally robust
  optimization.
\newblock {\em The Annals of Statistics}, 49(3):1378--1406, 2021.

\bibitem{ganin2015unsupervised}
Yaroslav Ganin and Victor Lempitsky.
\newblock Unsupervised domain adaptation by backpropagation.
\newblock In {\em International conference on machine learning}, pages
  1180--1189. PMLR, 2015.

\bibitem{geirhos2020shortcut}
Robert Geirhos, J{\"o}rn-Henrik Jacobsen, Claudio Michaelis, Richard Zemel,
  Wieland Brendel, Matthias Bethge, and Felix~A Wichmann.
\newblock Shortcut learning in deep neural networks.
\newblock {\em Nature Machine Intelligence}, 2(11):665--673, 2020.

\bibitem{geirhos2018imagenet}
Robert Geirhos, Patricia Rubisch, Claudio Michaelis, Matthias Bethge, Felix~A
  Wichmann, and Wieland Brendel.
\newblock Imagenet-trained cnns are biased towards texture; increasing shape
  bias improves accuracy and robustness.
\newblock In {\em International Conference on Learning Representations}, 2019.

\bibitem{gulrajani2020search}
Ishaan Gulrajani and David Lopez-Paz.
\newblock In search of lost domain generalization.
\newblock In {\em International Conference on Learning Representations}, 2020.

\bibitem{gururangan2018annotation}
Suchin Gururangan, Swabha Swayamdipta, Omer Levy, Roy Schwartz, Samuel~R
  Bowman, and Noah~A Smith.
\newblock Annotation artifacts in natural language inference data.
\newblock In {\em NAACL-HLT (2)}, 2018.

\bibitem{heinze2018invariant}
Christina Heinze-Deml, Jonas Peters, and Nicolai Meinshausen.
\newblock Invariant causal prediction for nonlinear models.
\newblock {\em Journal of Causal Inference}, 6(2), 2018.

\bibitem{jin2020domain}
Wengong Jin, Regina Barzilay, and Tommi Jaakkola.
\newblock Domain extrapolation via regret minimization.
\newblock {\em arXiv preprint arXiv:2006.03908}, 2020.

\bibitem{kamath2021does}
Pritish Kamath, Akilesh Tangella, Danica Sutherland, and Nathan Srebro.
\newblock Does invariant risk minimization capture invariance?
\newblock In {\em International Conference on Artificial Intelligence and
  Statistics}, pages 4069--4077. PMLR, 2021.

\bibitem{kingma2013auto}
Diederik~P Kingma and Max Welling.
\newblock Auto-encoding variational bayes.
\newblock {\em arXiv preprint arXiv:1312.6114}, 2013.

\bibitem{krueger2021out}
David Krueger, Ethan Caballero, Joern-Henrik Jacobsen, Amy Zhang, Jonathan
  Binas, Dinghuai Zhang, Remi Le~Priol, and Aaron Courville.
\newblock Out-of-distribution generalization via risk extrapolation (rex).
\newblock In {\em International Conference on Machine Learning}, pages
  5815--5826. PMLR, 2021.

\bibitem{lee2018minimax}
Jaeho Lee and Maxim Raginsky.
\newblock Minimax statistical learning with wasserstein distances.
\newblock In {\em Neural Information Processing Systems}, volume~31, 2018.

\bibitem{li2021invariant}
Bo~Li, Yifei Shen, Yezhen Wang, Wenzhen Zhu, Colorado~J Reed, Jun Zhang,
  Dongsheng Li, Kurt Keutzer, and Han Zhao.
\newblock Invariant information bottleneck for domain generalization.
\newblock In {\em Association for the Advancement of Artificial Intelligence},
  2022.

\bibitem{li2018deep}
Ya~Li, Xinmei Tian, Mingming Gong, Yajing Liu, Tongliang Liu, Kun Zhang, and
  Dacheng Tao.
\newblock Deep domain generalization via conditional invariant adversarial
  networks.
\newblock In {\em European Conference on Computer Vision}, pages 624--639,
  2018.

\bibitem{liu2021learning}
Chang Liu, Xinwei Sun, Jindong Wang, Haoyue Tang, Tao Li, Tao Qin, Wei Chen,
  and Tie-Yan Liu.
\newblock Learning causal semantic representation for out-of-distribution
  prediction.
\newblock In {\em Neural Information Processing Systems}, volume~34, 2021.

\bibitem{lu2021nonlinear}
Chaochao Lu, Yuhuai Wu, Jo{\'s}e~Miguel Hern{\'a}ndez-Lobato, and Bernhard
  Sch{\"o}lkopf.
\newblock Nonlinear invariant risk minimization: A causal approach.
\newblock {\em arXiv preprint arXiv:2102.12353}, 2021.

\bibitem{lu2022invariant}
Chaochao Lu, Yuhuai Wu, Jos{\'e}~Miguel Hern{\'a}ndez-Lobato, and Bernhard
  Sch{\"o}lkopf.
\newblock Invariant causal representation learning for out-of-distribution
  generalization.
\newblock In {\em International Conference on Learning Representations}, 2022.

\bibitem{mahajan2021domain}
Divyat Mahajan, Shruti Tople, and Amit Sharma.
\newblock Domain generalization using causal matching.
\newblock In {\em International Conference on Machine Learning}, pages
  7313--7324. PMLR, 2021.

\bibitem{muandet2013domain}
Krikamol Muandet, David Balduzzi, and Bernhard Sch{\"o}lkopf.
\newblock Domain generalization via invariant feature representation.
\newblock In {\em International Conference on Machine Learning}, pages 10--18.
  PMLR, 2013.

\bibitem{namkoong2016stochastic}
Hongseok Namkoong and John~C Duchi.
\newblock Stochastic gradient methods for distributionally robust optimization
  with f-divergences.
\newblock In {\em Neural Information processing Systems}, volume~29, 2016.

\bibitem{nguyen2015deep}
Anh Nguyen, Jason Yosinski, and Jeff Clune.
\newblock Deep neural networks are easily fooled: High confidence predictions
  for unrecognizable images.
\newblock In {\em Computer Vision and Pattern Recognition Conference}, pages
  427--436, 2015.

\bibitem{pearl2009causality}
Judea Pearl.
\newblock {\em Causality}.
\newblock Cambridge university press, 2009.

\bibitem{peters2016causal}
Jonas Peters, Peter B{\"u}hlmann, and Nicolai Meinshausen.
\newblock Causal inference by using invariant prediction: identification and
  confidence intervals.
\newblock {\em Journal of the Royal Statistical Society: Series B (Statistical
  Methodology)}, 78(5):947--1012, 2016.

\bibitem{peters2017elements}
Jonas Peters, Dominik Janzing, and Bernhard Sch{\"o}lkopf.
\newblock {\em Elements of causal inference: foundations and learning
  algorithms}.
\newblock The MIT Press, 2017.

\bibitem{pezeshki2021gradient}
Mohammad Pezeshki, Oumar Kaba, Yoshua Bengio, Aaron~C Courville, Doina Precup,
  and Guillaume Lajoie.
\newblock Gradient starvation: A learning proclivity in neural networks.
\newblock In {\em Neural Information Processing Systems}, volume~34, 2021.

\bibitem{rezende2014stochastic}
Danilo~Jimenez Rezende, Shakir Mohamed, and Daan Wierstra.
\newblock Stochastic backpropagation and approximate inference in deep
  generative models.
\newblock In {\em International Conference on Machine Learning}, pages
  1278--1286. PMLR, 2014.

\bibitem{rojas2018invariant}
Mateo Rojas-Carulla, Bernhard Sch{\"o}lkopf, Richard Turner, and Jonas Peters.
\newblock Invariant models for causal transfer learning.
\newblock {\em The Journal of Machine Learning Research}, 19(1):1309--1342,
  2018.

\bibitem{rosenfeld2018elephant}
Amir Rosenfeld, Richard Zemel, and John~K Tsotsos.
\newblock The elephant in the room.
\newblock {\em arXiv preprint arXiv:1808.03305}, 2018.

\bibitem{rosenfeld2020risks}
Elan Rosenfeld, Pradeep~Kumar Ravikumar, and Andrej Risteski.
\newblock The risks of invariant risk minimization.
\newblock In {\em International Conference on Learning Representations}, 2021.

\bibitem{scholkopf2021toward}
Bernhard Sch{\"o}lkopf, Francesco Locatello, Stefan Bauer, Nan~Rosemary Ke, Nal
  Kalchbrenner, Anirudh Goyal, and Yoshua Bengio.
\newblock Toward causal representation learning.
\newblock {\em Proceedings of the IEEE}, 109(5):612--634, 2021.

\bibitem{sinha2017certifying}
Aman Sinha, Hongseok Namkoong, Riccardo Volpi, and John Duchi.
\newblock Certifying some distributional robustness with principled adversarial
  training.
\newblock {\em arXiv preprint arXiv:1710.10571}, 2017.

\bibitem{soudry2018implicit}
Daniel Soudry, Elad Hoffer, Mor~Shpigel Nacson, Suriya Gunasekar, and Nathan
  Srebro.
\newblock The implicit bias of gradient descent on separable data.
\newblock {\em The Journal of Machine Learning Research}, 19(1):2822--2878,
  2018.

\bibitem{szegedy2013intriguing}
Christian Szegedy, Wojciech Zaremba, Ilya Sutskever, Joan Bruna, Dumitru Erhan,
  Ian Goodfellow, and Rob Fergus.
\newblock Intriguing properties of neural networks.
\newblock {\em arXiv preprint arXiv:1312.6199}, 2013.

\bibitem{thomas2006elements}
MTCAJ Thomas and A~Thomas Joy.
\newblock {\em Elements of information theory}.
\newblock Wiley-Interscience, 2006.

\bibitem{tishby1999information}
N~TISHBY.
\newblock The information bottleneck method.
\newblock In {\em Annual Allerton Conference on Communications, Control and
  Computing}, pages 368--377, 1999.

\bibitem{wang2019learning}
Haohan Wang, Songwei Ge, Zachary Lipton, and Eric~P Xing.
\newblock Learning robust global representations by penalizing local predictive
  power.
\newblock In {\em Neural Information Processing Systems}, volume~32, 2019.

\bibitem{InvariantSubDG}
Haoxiang Wang, Haozhe Si, Bo~Li, and Han Zhao.
\newblock Provable domain generalization via invariant-feature subspace
  recovery.
\newblock In {\em International Conference on Machine Learning}, 2022.

\bibitem{wiles2022a}
Olivia Wiles, Sven Gowal, Florian Stimberg, Sylvestre-Alvise Rebuffi, Ira
  Ktena, Krishnamurthy~Dj Dvijotham, and Ali~Taylan Cemgil.
\newblock A fine-grained analysis on distribution shift.
\newblock In {\em International Conference on Learning Representations}, 2022.

\bibitem{xie2020risk}
Chuanlong Xie, Fei Chen, Yue Liu, and Zhenguo Li.
\newblock Risk variance penalization: From distributional robustness to
  causality.
\newblock {\em arXiv preprint arXiv:2006.07544}, 1, 2020.

\bibitem{xu2021fourier}
Qinwei Xu, Ruipeng Zhang, Ya~Zhang, Yanfeng Wang, and Qi~Tian.
\newblock A fourier-based framework for domain generalization.
\newblock In {\em IEEE/CVF Conference on Computer Vision and Pattern
  Recognition}, pages 14383--14392, 2021.

\bibitem{ye2021towards}
Haotian Ye, Chuanlong Xie, Tianle Cai, Ruichen Li, Zhenguo Li, and Liwei Wang.
\newblock Towards a theoretical framework of out-of-distribution
  generalization.
\newblock In {\em Neural Information Processing Systems}, 2021.

\bibitem{zhang2021deep}
Xingxuan Zhang, Peng Cui, Renzhe Xu, Linjun Zhou, Yue He, and Zheyan Shen.
\newblock Deep stable learning for out-of-distribution generalization.
\newblock In {\em IEEE/CVF Conference on Computer Vision and Pattern
  Recognition}, pages 5372--5382, 2021.

\bibitem{zhao2019learning}
Han Zhao, Remi~Tachet Des~Combes, Kun Zhang, and Geoffrey Gordon.
\newblock On learning invariant representations for domain adaptation.
\newblock In {\em International Conference on Machine Learning}, pages
  7523--7532. PMLR, 2019.

\bibitem{zhao2020domain}
Shanshan Zhao, Mingming Gong, Tongliang Liu, Huan Fu, and Dacheng Tao.
\newblock Domain generalization via entropy regularization.
\newblock In {\em Neural Information Processing Systems}, volume~33, pages
  16096--16107, 2020.

\bibitem{zhou2020domain}
Kaiyang Zhou, Yongxin Yang, Yu~Qiao, and Tao Xiang.
\newblock Domain generalization with mixstyle.
\newblock In {\em International Conference on Learning Representations}, 2021.

\end{thebibliography}
%%%%%%%%%%%%%%%%%%%%%%%%%%%%%%%%%%%%%%%%%%%%%%%%%%%%%%%%%%%%%%%%%%%%%%%%%%%%%%%%%%%%%%%%%%%%%%%%%%%%%%%%%%%%%%%%%%%%%%%%%%%
\newpage
%%%%%%%%%%%%%%%%%%%%%%%%%%%%%%%%%%%%%%%%%%%%%%%%%%%%%%%%%%%%

\appendix

\newtheorem{appendix-definition}{Definition}[section]
\newtheorem{appendix-proposition}{Proposition}[section]
\newtheorem{appendix-theorem}{Theorem}[section]
\newtheorem{appendix-lemma}{Lemma}[section]
\newtheorem{appendix-assumption}{Assumption}[section]
\newtheorem{appendix-corollary}{Corollary}[section]
\newtheorem{appendix-example}{Example}[section]

\setcounter{table}{0}
\renewcommand{\thetable}{A\arabic{table}}
\setcounter{figure}{0}
\renewcommand{\thefigure}{A\arabic{figure}}

\section*{Appendix}

\section[\appendixname~\thesection]{Experiments details}
In this section, we provide more details on the experiments. The code to reproduce the experiments can be found at \url{https://github.com/szubing/CSIB}.
\subsection[\appendixname~\thesubsection]{Optimization loss of IB-ERM}
The objective function of IB-ERM is as follow:
\begin{equation}\label{appendix-IBERM}
 \quad \min_{w, \Phi} \sum_{e\in\mathcal{E}_{tr}} h^e(\Phi) \quad \text{s.t.} \ \frac{1}{|\mathcal{E}_{tr}|}\sum_{e\in \mathcal{E}_{tr}} R^e(w\circ\Phi) \leq r^{th}.
\end{equation}
Since the entropy of $h^e(\Phi) = H(\Phi(X^e))$ is hard to estimate by a differential variable that can be optimized by using gradient descent, we follow \cite{ahuja2021invariance} by using the variance instead of the entropy for optimization. The total loss function is given by
\begin{equation}\label{appendix-IBERMloss}
loss(w, \Phi) = \sum_{e\in \mathcal{E}_{tr}}(R^e(w\circ\Phi) + \lambda \text{Var}(\Phi))
\end{equation}
with a hyperparameter $\lambda$ onto it.

% \begin{table}[H]
% \caption{This is a table caption.\label{tab5}}
% \newcolumntype{C}{>{\centering\arraybackslash}X}
% \begin{tabularx}{\textwidth}{CCC}
% \toprule
% \textbf{Title 1}	& \textbf{Title 2}	& \textbf{Title 3}\\
% \midrule
% Entry 1		& Data			& Data\\
% Entry 2		& Data			& Data\\
% \bottomrule
% \end{tabularx}
% \end{table}

\subsection[\appendixname~\thesubsection]{Experiments setup}
\textbf{Model, hyperparameters, loss, and evaluation}. In all experiments, we follow the same protocol as prescribed by \cite{aubin2021linear, ahuja2021invariance} for the model/hyperparameter selection, training, and evaluation. Except those specified, for all experiments across three Examples and five comparing methods, the model is the same with a linear feature extractor $\Phi\in\mathbb{R}^{d\times d}$ followed by a linear classifier $w\in\mathbb{R}^{d+1}$. We use binary cross-entropy loss for classification. All hyperparameters, including the learning rate, the penalty term in IRM, or the $\lambda$ associated with the Var$(\Phi)$ in Equation (\ref{appendix-IBERMloss}), etc., are randomly searched and selected by using 20 test samples for validation. The results reported in the main manuscript use 3 hyperparameter queries of each and average over 5 data seeds. The results when searching over more hyperparameter values are reported in the supplementary experiments. The search spaces of all the hyperparameters follow the same as in \cite{aubin2021linear, ahuja2021invariance}. The classification test errors between 0 and 1 are reported.

\textbf{Compute description}. Our computing resource is one GPU of NVIDIA GeForce GTX 1080 Ti with 6 CPU cores of Intel(R) Core(TM) i7-8700 CPU @ 3.20GHz.

\textbf{Existing codes and datasets used.} In our experiments, we mainly rely on the following two github repositories: InvarianceUnitTests\footnote{https://github.com/facebookresearch/InvarianceUnitTests} and IB-IRM\footnote{https://github.com/ahujak/IB-IRM}.

\subsection[\appendixname~\thesubsection]{Supplementary experiments}\label{appendix-experiments}
The purpose of the first supplementary experiment is to illustrate what the result would be when we increase the number of running seeds in the hyperparameters selection. These results are shown in Table \ref{appendix-table-main_results}, where we increase the number of hyperparameter queries to 10 of each. It is clear that in overall, the results of CSIB in Table \ref{appendix-table-main_results} are much better and have less fluctuations than those in Table \ref{table-main_results}, and the conclusions remain almost the same as we have summarized in section \ref{section-summaryofresults}. This further verifies the effectiveness of CSIB method.

\begin{table}
\caption{Supplementary results when using 10 hyperparameter queries. \#Envs means the number of training environments, and (min) reports the minimal test classification error across different running data seeds.}\label{appendix-table-main_results}
\centering
\adjustbox{max width=\textwidth}{%
% \begin{adjustwidth}{-\extralength}{0cm}
\begin{tabular}{lccccccc}
\toprule
            &\#Envs    & ERM (min)       & IRM (min)         & IB-ERM (min)         & IB-IRM (min)         & CSIB (min)                     & Oracle (min)         \\
\midrule
Example 1   & 1   & 0.50 $\pm$ 0.01 (0.49) & 0.50 $\pm$ 0.01 (0.49) & 0.23 $\pm$ 0.02 (0.22) & 0.31 $\pm$ 0.10 (0.25) & 0.23 $\pm$ 0.02 (0.22)  & 0.00 $\pm$ 0.00 (0.00) \\
Example 1S  & 1   & 0.50 $\pm$ 0.00 (0.49) & 0.50 $\pm$ 0.00 (0.49) & 0.09 $\pm$ 0.04 (0.04) & 0.30 $\pm$ 0.10 (0.25) & 0.08 $\pm$ 0.04 (0.04)  & 0.00 $\pm$ 0.00 (0.00) \\
Example 2   & 1   & 0.40 $\pm$ 0.20 (0.00) & 0.00 $\pm$ 0.00 (0.00) & 0.50 $\pm$ 0.00 (0.49) & 0.48 $\pm$ 0.03 (0.43) & 0.00 $\pm$ 0.00 (0.00)  & 0.00 $\pm$ 0.00 (0.00) \\
Example 2S  & 1   & 0.50 $\pm$ 0.00 (0.50) & 0.30 $\pm$ 0.25 (0.00) & 0.50 $\pm$ 0.00 (0.50) & 0.50 $\pm$ 0.01 (0.48) & 0.00 $\pm$ 0.00 (0.00)  & 0.00 $\pm$ 0.00 (0.00) \\
Example 3   & 1   & 0.16 $\pm$ 0.06 (0.09) & 0.03 $\pm$ 0.00 (0.03) & 0.50 $\pm$ 0.01 (0.49) & 0.41 $\pm$ 0.09 (0.25) & 0.02 $\pm$ 0.01 (0.00)  & 0.00 $\pm$ 0.00 (0.00) \\
Example 3S  & 1   & 0.16 $\pm$ 0.06 (0.10) & 0.04 $\pm$ 0.01 (0.02) & 0.50 $\pm$ 0.00 (0.50) & 0.41 $\pm$ 0.12 (0.26) & 0.01 $\pm$ 0.01 (0.00)  & 0.00 $\pm$ 0.00 (0.00) \\ \hline
Example 1   & 3   & 0.44 $\pm$ 0.01 (0.44) & 0.44 $\pm$ 0.01 (0.44) & 0.21 $\pm$ 0.00 (0.21) & 0.21 $\pm$ 0.10 (0.06) & 0.21 $\pm$ 0.00 (0.21)  & 0.00 $\pm$ 0.00 (0.00) \\
Example 1S  & 3   & 0.45 $\pm$ 0.00 (0.44) & 0.45 $\pm$ 0.00 (0.44) & 0.09 $\pm$ 0.03 (0.05) & 0.23 $\pm$ 0.13 (0.01) & 0.09 $\pm$ 0.03 (0.05)  & 0.00 $\pm$ 0.00 (0.00) \\
Example 2   & 3   & 0.13 $\pm$ 0.07 (0.00) & 0.00 $\pm$ 0.00 (0.00) & 0.50 $\pm$ 0.00 (0.50) & 0.33 $\pm$ 0.04 (0.25) & 0.00 $\pm$ 0.00 (0.00)  & 0.00 $\pm$ 0.00 (0.00) \\
Example 2S  & 3   & 0.50 $\pm$ 0.00 (0.50) & 0.14 $\pm$ 0.20 (0.00) & 0.50 $\pm$ 0.00 (0.50) & 0.34 $\pm$ 0.01 (0.33) & 0.00 $\pm$ 0.00 (0.00)  & 0.00 $\pm$ 0.00 (0.00) \\
Example 3   & 3   & 0.17 $\pm$ 0.04 (0.14) & 0.02 $\pm$ 0.00 (0.02) & 0.50 $\pm$ 0.01 (0.49) & 0.43 $\pm$ 0.08 (0.29) & 0.01 $\pm$ 0.00 (0.00)  & 0.00 $\pm$ 0.00 (0.00) \\
Example 3S  & 3   & 0.17 $\pm$ 0.04 (0.13) & 0.02 $\pm$ 0.00 (0.02) & 0.50 $\pm$ 0.00 (0.50) & 0.36 $\pm$ 0.18 (0.07) & 0.01 $\pm$ 0.00 (0.00)  & 0.00 $\pm$ 0.00 (0.00) \\ \hline
Example 1   & 6   & 0.46 $\pm$ 0.01 (0.44) & 0.46 $\pm$ 0.09 (0.41) & 0.22 $\pm$ 0.01 (0.21) & 0.41 $\pm$ 0.11 (0.26) & 0.22 $\pm$ 0.01 (0.21)  & 0.00 $\pm$ 0.00 (0.00) \\
Example 1S  & 6   & 0.46 $\pm$ 0.02 (0.44) & 0.46 $\pm$ 0.02 (0.44) & 0.06 $\pm$ 0.04 (0.02) & 0.45 $\pm$ 0.07 (0.41) & 0.06 $\pm$ 0.04 (0.02)  & 0.00 $\pm$ 0.00 (0.00) \\
Example 2   & 6   & 0.21 $\pm$ 0.03 (0.17) & 0.00 $\pm$ 0.00 (0.00) & 0.50 $\pm$ 0.00 (0.50) & 0.36 $\pm$ 0.03 (0.31) & 0.00 $\pm$ 0.00 (0.00)  & 0.00 $\pm$ 0.00 (0.00) \\
Example 2S  & 6   & 0.50 $\pm$ 0.00 (0.50) & 0.10 $\pm$ 0.20 (0.00) & 0.50 $\pm$ 0.00 (0.50) & 0.19 $\pm$ 0.16 (0.01) & 0.00 $\pm$ 0.00 (0.00)  & 0.00 $\pm$ 0.00 (0.00) \\
Example 3   & 6   & 0.17 $\pm$ 0.03 (0.14) & 0.02 $\pm$ 0.00 (0.02) & 0.50 $\pm$ 0.00 (0.49) & 0.37 $\pm$ 0.16 (0.16) & 0.01 $\pm$ 0.00 (0.00)  & 0.00 $\pm$ 0.00 (0.00) \\
Example 3S  & 6   & 0.17 $\pm$ 0.03 (0.14) & 0.02 $\pm$ 0.00 (0.02) & 0.50 $\pm$ 0.00 (0.50) & 0.46 $\pm$ 0.09 (0.28) & 0.01 $\pm$ 0.00 (0.00)  & 0.00 $\pm$ 0.00 (0.00) \\
\bottomrule
\end{tabular}
% \end{adjustwidth}
}
\end{table}

\textbf{Observation on different settings in Example 1/1S.} In our main experiments of Example 1/1S, we set $p^e=1$ and $q=0$ to make the spurious features and the invariant features both linearly separable to confuse each other. Here, we analyse what the result would be if we vary the values of them. Following \cite{aubin2021linear}, we set $p^{e_0}=0.95$, $p^{e_1}=0.97$, $p^{e_2}=0.99$, and $p^{e_j}\sim \text{Uniform}(0.9,1)$ to make spurious features linearly inseparable, and $q$ is set to 0/0.05 to make invariant features linearly separable/inseparable. Table \ref{appendix-table-example1} shows the corresponding results. Interestingly, we find that all methods except for IB-IRM have ideal error rate (the same as the Oracle) when the spurious features are linearly inseparable ($p^e\neq 1$), even when the invariant features are linearly inseparable too ($q=0.05$). Why would this happen? We then remove the linear embedding $\Phi$, the results are presented in Table \ref{appendix-table-example1-withoutbackbone}. Comparing the results between Tables \ref{appendix-table-example1} and \ref{appendix-table-example1-withoutbackbone}, we found there is a significant inductive bias of neural network, though the model is linear. Further analysis to such observation is out of scope of this paper, but this would be an interesting avenue for future research.

\begin{table}
\caption{Results in Example 1/1S, where the learning model is a linear embedding $\Phi\in\mathbb{R}^{d\times d}$ followed by a linear classifier $w\in \mathbb{R}^{d+1}$. }
\begin{center}
\adjustbox{max width=\textwidth}{%
\begin{tabular}{lccccccccc}
\toprule
     & \#Envs   & $p^e=1$? & q       & ERM             & IB-ERM          & IB-IRM          & CSIB             & IRM             & Oracle          \\
\midrule
Example 1   & 1   & Yes & 0      & 0.50 $\pm$ 0.01 & 0.23 $\pm$ 0.02 & 0.31 $\pm$ 0.10 & 0.23 $\pm$ 0.02 & 0.50 $\pm$ 0.01 & 0.00 $\pm$ 0.00 \\
Example 1S  & 1   & Yes & 0      & 0.50 $\pm$ 0.00 & 0.46 $\pm$ 0.04 & 0.30 $\pm$ 0.10 & 0.46 $\pm$ 0.04 & 0.50 $\pm$ 0.00 & 0.00 $\pm$ 0.00 \\
Example 1   & 3   & Yes & 0      & 0.45 $\pm$ 0.01 & 0.22 $\pm$ 0.01 & 0.23 $\pm$ 0.13 & 0.22 $\pm$ 0.01 & 0.45 $\pm$ 0.01 & 0.00 $\pm$ 0.00 \\
Example 1S  & 3   & Yes & 0      & 0.45 $\pm$ 0.00 & 0.41 $\pm$ 0.04 & 0.27 $\pm$ 0.11 & 0.41 $\pm$ 0.04 & 0.45 $\pm$ 0.00 & 0.00 $\pm$ 0.00 \\
Example 1   & 6   & Yes & 0      & 0.46 $\pm$ 0.01 & 0.22 $\pm$ 0.01 & 0.37 $\pm$ 0.14 & 0.22 $\pm$ 0.01 & 0.46 $\pm$ 0.09 & 0.00 $\pm$ 0.00 \\
Example 1S  & 6   & Yes & 0      & 0.46 $\pm$ 0.02 & 0.35 $\pm$ 0.10 & 0.42 $\pm$ 0.12 & 0.35 $\pm$ 0.10 & 0.46 $\pm$ 0.02 & 0.00 $\pm$ 0.00 \\
\hline
Example 1   & 1   & No & 0      & 0.00 $\pm$ 0.00 & 0.00 $\pm$ 0.00 & 0.15 $\pm$ 0.20 & 0.00 $\pm$ 0.00 & 0.00 $\pm$ 0.00 & 0.00 $\pm$ 0.00 \\
Example 1S  & 1   & No & 0      & 0.00 $\pm$ 0.00 & 0.00 $\pm$ 0.00 & 0.12 $\pm$ 0.19 & 0.00 $\pm$ 0.00 & 0.00 $\pm$ 0.00 & 0.00 $\pm$ 0.00 \\
Example 1   & 3   & No & 0      & 0.00 $\pm$ 0.00 & 0.00 $\pm$ 0.00 & 0.00 $\pm$ 0.00 & 0.00 $\pm$ 0.00 & 0.00 $\pm$ 0.00 & 0.00 $\pm$ 0.00 \\
Example 1S  & 3   & No & 0      & 0.00 $\pm$ 0.00 & 0.00 $\pm$ 0.00 & 0.00 $\pm$ 0.01 & 0.00 $\pm$ 0.00 & 0.00 $\pm$ 0.00 & 0.00 $\pm$ 0.00 \\
Example 1   & 6   & No & 0      & 0.00 $\pm$ 0.00 & 0.00 $\pm$ 0.00 & 0.30 $\pm$ 0.20 & 0.00 $\pm$ 0.00 & 0.00 $\pm$ 0.00 & 0.00 $\pm$ 0.00 \\
Example 1S  & 6   & No & 0      & 0.00 $\pm$ 0.00 & 0.00 $\pm$ 0.00 & 0.31 $\pm$ 0.20 & 0.00 $\pm$ 0.00 & 0.04 $\pm$ 0.06 & 0.00 $\pm$ 0.00 \\
\hline
Example 1   & 1   & No & 0.05      & 0.05 $\pm$ 0.00 & 0.05 $\pm$ 0.00 & 0.32 $\pm$ 0.22 & 0.05 $\pm$ 0.00 & 0.05 $\pm$ 0.00 & 0.05 $\pm$ 0.00 \\
Example 1S   & 1   & No & 0.05     & 0.05 $\pm$ 0.00 & 0.05 $\pm$ 0.00 & 0.19 $\pm$ 0.17 & 0.05 $\pm$ 0.00 & 0.05 $\pm$ 0.00 & 0.05 $\pm$ 0.00 \\
Example 1   & 3   & No & 0.05      & 0.05 $\pm$ 0.00 & 0.05 $\pm$ 0.00 & 0.07 $\pm$ 0.03 & 0.05 $\pm$ 0.00 & 0.05 $\pm$ 0.00 & 0.05 $\pm$ 0.00 \\
Example 1S   & 3   & No & 0.05     & 0.05 $\pm$ 0.00 & 0.05 $\pm$ 0.00 & 0.05 $\pm$ 0.00 & 0.05 $\pm$ 0.00 & 0.05 $\pm$ 0.00 & 0.05 $\pm$ 0.00 \\
Example 1   & 6   & No & 0.05      & 0.05 $\pm$ 0.00 & 0.05 $\pm$ 0.00 & 0.30 $\pm$ 0.21 & 0.05 $\pm$ 0.00 & 0.05 $\pm$ 0.00 & 0.05 $\pm$ 0.00 \\
Example 1S   & 6   & No & 0.05     & 0.05 $\pm$ 0.00 & 0.05 $\pm$ 0.00 & 0.32 $\pm$ 0.19 & 0.05 $\pm$ 0.00 & 0.05 $\pm$ 0.00 & 0.05 $\pm$ 0.00 \\
\bottomrule
\end{tabular}}
\end{center}
\label{appendix-table-example1}
\end{table}

\begin{table}
\caption{Results in Example 1/1S, where the learning model is a linear classifier $w\in \mathbb{R}^{d+1}$ without linear embedding $\Phi$. CSIB must requires a feature extractor, so there are not results related to CSIB.}
\begin{center}
\adjustbox{max width=\textwidth}{%
\begin{tabular}{lccccccccc}
\toprule
     & \#Envs   & $p^e=1$? & q       & ERM             & IB-ERM          & IB-IRM                & IRM             & Oracle          \\
\midrule
Example 1    & 1   & Yes & 0      & 0.50 $\pm$ 0.01 & 0.25 $\pm$ 0.01 & 0.31 $\pm$ 0.10 & 0.50 $\pm$ 0.01 & 0.00 $\pm$ 0.00 \\
Example 1S   & 1   & Yes & 0    & 0.50 $\pm$ 0.00 & 0.49 $\pm$ 0.01 & 0.30 $\pm$ 0.10 & 0.50 $\pm$ 0.00 & 0.00 $\pm$ 0.00 \\
Example 1    & 3   & Yes & 0      & 0.44 $\pm$ 0.01 & 0.23 $\pm$ 0.01 & 0.21 $\pm$ 0.10 & 0.44 $\pm$ 0.01 & 0.00 $\pm$ 0.00 \\
Example 1S   & 3   & Yes & 0     & 0.45 $\pm$ 0.00 & 0.44 $\pm$ 0.01 & 0.42 $\pm$ 0.04 & 0.45 $\pm$ 0.00 & 0.00 $\pm$ 0.00 \\
Example 1    & 6   & Yes & 0      & 0.46 $\pm$ 0.01 & 0.27 $\pm$ 0.07 & 0.41 $\pm$ 0.11 & 0.46 $\pm$ 0.01 & 0.01 $\pm$ 0.01 \\
Example 1S   & 6   & Yes & 0     & 0.46 $\pm$ 0.02 & 0.42 $\pm$ 0.08 & 0.46 $\pm$ 0.09 & 0.46 $\pm$ 0.02 & 0.01 $\pm$ 0.02 \\
\hline
Example 1    & 1   & No & 0      & 0.50 $\pm$ 0.01 & 0.00 $\pm$ 0.00 & 0.15 $\pm$ 0.20 & 0.50 $\pm$ 0.01 & 0.00 $\pm$ 0.00 \\
Example 1S   & 1   & No & 0      & 0.50 $\pm$ 0.00 & 0.00 $\pm$ 0.00 & 0.13 $\pm$ 0.19 & 0.50 $\pm$ 0.00 & 0.00 $\pm$ 0.00 \\
Example 1    & 3   & No & 0      & 0.45 $\pm$ 0.01 & 0.00 $\pm$ 0.00 & 0.00 $\pm$ 0.00 & 0.45 $\pm$ 0.01 & 0.00 $\pm$ 0.00 \\
Example 1S   & 3   & No & 0     & 0.45 $\pm$ 0.00 & 0.01 $\pm$ 0.02 & 0.08 $\pm$ 0.14 & 0.46 $\pm$ 0.02 & 0.00 $\pm$ 0.00 \\
Example 1    & 6   & No & 0      & 0.46 $\pm$ 0.01 & 0.10 $\pm$ 0.16 & 0.30 $\pm$ 0.20 & 0.46 $\pm$ 0.01 & 0.01 $\pm$ 0.01 \\
Example 1S   & 6   & No & 0    & 0.46 $\pm$ 0.01 & 0.24 $\pm$ 0.19 & 0.41 $\pm$ 0.12 & 0.47 $\pm$ 0.03 & 0.01 $\pm$ 0.02 \\
\hline
Example 1    & 1   & No & 0.05      & 0.50 $\pm$ 0.01 & 0.05 $\pm$ 0.00 & 0.32 $\pm$ 0.22 & 0.50 $\pm$ 0.01 & 0.05 $\pm$ 0.00 \\
Example 1S   & 1   & No & 0.05      & 0.50 $\pm$ 0.01 & 0.05 $\pm$ 0.01 & 0.20 $\pm$ 0.17 & 0.50 $\pm$ 0.00 & 0.05 $\pm$ 0.00 \\
Example 1    & 3   & No & 0.05     & 0.45 $\pm$ 0.01 & 0.05 $\pm$ 0.00 & 0.07 $\pm$ 0.03 & 0.47 $\pm$ 0.01 & 0.05 $\pm$ 0.00 \\
Example 1S   & 3   & No & 0.05     & 0.45 $\pm$ 0.01 & 0.07 $\pm$ 0.03 & 0.11 $\pm$ 0.11 & 0.46 $\pm$ 0.01 & 0.05 $\pm$ 0.00 \\
Example 1    & 6   & No & 0.05      & 0.47 $\pm$ 0.01 & 0.14 $\pm$ 0.14 & 0.30 $\pm$ 0.21 & 0.47 $\pm$ 0.01 & 0.05 $\pm$ 0.00 \\
Example 1S   & 6   & No & 0.05     & 0.47 $\pm$ 0.01 & 0.27 $\pm$ 0.18 & 0.42 $\pm$ 0.11 & 0.47 $\pm$ 0.01 & 0.05 $\pm$ 0.01 \\
\bottomrule
\end{tabular}}
\end{center}
\label{appendix-table-example1-withoutbackbone}
\end{table}

\textbf{Observation on linearly separable properties of high-dimensional data.} In here, we empirically show that for $o$-dimensional data, we have high probability that $o$ randomly drawn points are linearly separable for any two subsets. To verify that, we design a random experiment as follows: (1) Let $o\in[100, 10000]$, and we randomly drawn $o$ points from $[-1, 1]^o$, and give random labels to these $o$ points of 0 or 1; (2) We train a linear classifier to fit these $o$ points and report the final training error; (3) Do (1) and (2) 100 times for different seeds. Our results show that for 100 runs, all training errors reach to 0 for every $o$, which proves our conjecture.

Then, we look back to the Theorem \ref{theorem-analysis2}. For real data like image, the dimension of spurious features $o$ is often high. Assume different environments enjoy different spurious points randomly, then from the above observation, there is a high probability that the following events will occur: For any labeling data in the $n$ training environments with $n<o/2$ (2 is due to binary label), models could achieve zero training error by relying on spurious features only. This illustrates why prior methods easily fail to address OOD generalization under the Assumption \ref{assp-mode-3}.

\section[\appendixname~\thesection]{Proofs}
\subsection[\appendixname~\thesubsection]{Preliminary}\label{appendix-preliminary}
Before our proofs, we first review some useful properties related to the entropy \cite{thomas2006elements, ahuja2021invariance}.

\textbf{Entropy.} For discrete random variable $X\sim \mathbb{P}_X$ with support $\mathcal{X}$, its entropy (Shannon entropy) is defined as
\begin{equation}
H(X) = -\sum_{x\in\mathcal{X}} \mathbb{P}_X(X=x)\log (\mathbb{P}_X(X=x))
\end{equation}
The differential entropy of the continuous random variable $X\sim \mathbb{P}_X$ with support $\mathcal{X}$ is given by
\begin{equation}
h(X) = -\int_{x\in\mathcal{X}}p_X(x)\log (p_X(x))dx,
\end{equation}
where $p_X(x)$ is the probability density function of the distribution $\mathbb{P}_X$. Sometimes we may confuse using $H(X)$ or $h(X)$ to represent its entropy no matter $X$ is discrete or continuous.
\begin{Lemma}\label{appendix-lemma-iib-1}
If $X$ and $Y$ are discrete random variables that are independent, then
\begin{equation}
H(X+Y) \geq \max\{H(X), H(Y)\}.
\end{equation}
\end{Lemma}
\begin{proof}
Define $Z=X+Y$. Since $X\perp Y$, we have
\begin{align*}
H(Z|X) &= -\sum_{x\in\mathcal{X}}\mathbb{P}_X(x)\sum_{z\in\mathcal{Z}}\mathbb{P}_{Z|X}(Z=z|X=x)\log (\mathbb{P}_{Z|X}(Z=z|X=x)) \\
 &=  -\sum_{x\in\mathcal{X}}\mathbb{P}_X(x)\sum_{z\in\mathcal{Z}}\mathbb{P}_{Y|X}(Y=z-x|X=x)\log (\mathbb{P}_{Y|X}(Y=z-x|X=x)) \\
 &= -\sum_{x\in\mathcal{X}}\mathbb{P}_X(x)\sum_{z\in\mathcal{Z}}\mathbb{P}_{Y}(Y=z-x)\log (\mathbb{P}_{Y}(Y=z-x))\\
 &= -\sum_{x\in\mathcal{X}}\mathbb{P}_X(x)\sum_{y\in\mathcal{Y}}\mathbb{P}_{Y}(Y=y)\log (\mathbb{P}_{Y}(Y=y))\\
 &= H(Y),
\end{align*}
and similar we have $H(Z|Y)=H(X)$. Therefore,
\begin{align}
H(X+Y) = I(Z, X) + H(Z|X) = I(Z,X) + H(Y) \geq H(Y) \\
H(X+Y) = I(Z, Y) + H(Z|Y) = I(Z,Y) + H(X) \geq H(X).  \label{appendix-eq-1}
\end{align}
This completes the proof.
\end{proof}

\begin{Lemma}\label{appendix-lemma-iib-2}
If $X$ and $Y$ are continuous random variables that are independent, then
\begin{equation}
h(X+Y) \geq \max\{h(X), h(Y)\}.
\end{equation}
\end{Lemma}
\begin{proof}
Define $Z=X+Y$. Since $X\perp Y$, we have
\begin{align*}
h(Z|X) &= -\int_{x\in\mathcal{X}}p_X(x)\int_{z\in\mathcal{Z}}p_{Z|X}(Z=z|X=x)\log (p_{Z|X}(Z=z|X=x))dxdz \\
 &=  -\int_{x\in\mathcal{X}}p_X(x)\int_{z\in\mathcal{Z}}p_{Y|X}(Y=z-x|X=x)\log (p_{Y|X}(Y=z-x|X=x))dxdz \\
 &= -\int_{x\in\mathcal{X}}p_X(x)\int_{z\in\mathcal{Z}}p_{Y}(Y=z-x)\log (p_{Y}(Y=z-x))dxdz \\
 &= -\int_{x\in\mathcal{X}}p_X(x)dx\int_{y\in\mathcal{Y}}p_{Y}(Y=y)\log (p_{Y}(Y=y))dy \\
 &= h(Y),
\end{align*}
and similar we have $h(Z|Y)=h(X)$. Therefore,
\begin{align}
h(X+Y) = I(Z, X) + h(Z|X) = I(Z,X) + h(Y) \geq h(Y) \\
h(X+Y) = I(Z, Y) + h(Z|Y) = I(Z,Y) + h(X) \geq h(X). \label{appendix-eq-2}
\end{align}
This completes the proof.
\end{proof}

\begin{Lemma}\label{appendix-lemma-iib-3}
If $X$ and $Y$ are discrete random variables that are independent with the supports satisfying $2\leq|\mathcal{X}|<\infty, 2\leq|\mathcal{Y}|<\infty$, then
\begin{equation}
H(X+Y) > \max\{H(X), H(Y)\}.
\end{equation}
\end{Lemma}
\begin{proof}
From Lemma \ref{appendix-lemma-iib-1} and due to the symmetry of $X$ and $Y$, we only need to prove $H(X+Y) \neq H(X)$. The proof is by contradiction. Suppose $H(X+Y) = H(X)$, then from Equation \ref{appendix-eq-1} follows that $I(X+Y, Y) = 0$, thus $X+Y\perp Y$. However, $\mathbb{P}(Y=y_{max}|X+Y=x_{max}+y_{max}) = 1$, which is different from $\mathbb{P}(Y=y_{max}) < 1$ (due to $|Y|\geq2$). This contradicts $X+Y\perp Y$.
\end{proof}

\begin{Lemma}\label{appendix-lemma-iib-4}
If $X$ and $Y$ are continuous random variables that are independent and have a bounded support, then
\begin{equation}
h(X+Y) > \max\{h(X), h(Y)\}.
\end{equation}
\end{Lemma}
\begin{proof}
From Lemma \ref{appendix-lemma-iib-2} and due to the symmetry of $X$ and $Y$, we only need to prove $h(X+Y) \neq h(X)$. The proof is by contradiction. Suppose $h(X+Y) = h(X)$, then from Equation \ref{appendix-eq-2} follows that $I(X+Y, Y) = 0$, thus $X+Y\perp Y$. For any $\delta>0$, define an event $\mathcal{M}: x_{max} + y_{max} - \delta \leq X+Y \leq x_{max} + y_{max}$. If $\mathcal{M}$ occurs, then $Y\geq y_{max}-\delta$ and $X\geq x_{max}-\delta$. Thus, $\mathbb{P}_Y(Y\leq y_{max}-\delta|\mathcal{M}) = 0$. However, we can always choose a $\delta > 0$ that is small enough to make $\mathbb{P}_Y(Y\leq y_{max}-\delta) > 0$. This contradicts $X+Y\perp Y$.
\end{proof}

\subsection[\appendixname~\thesubsection]{Proof of Theorem \ref{theorem-analysis2}}
% We state the Theorem \ref{theorem-analysis2} here for convenience.
% \begin{Theorem}\label{appendix-theorem-analysis2}
% Following Assumption \ref{assp-linearSEM} or \ref{assp-linearSEM_PIIF}, if two sets $\cup_{e\in\mathcal{E}_{tr}}\mathcal{Z}^e_{spu}(Y^e=1)$ and $\cup_{e\in \mathcal{E}_{tr}}\mathcal{Z}^e_{spu}(Y^e=-1/0)$ are linearly separable and $H(Z^e_{inv})>H(Z^e_{spu})$ on each training environment $e$, then IB-IRM (and IRM, ERM, or IB-ERM) with any $r^{th}\in \mathbb{R}$ fails to address the OOD generalization problem (Equation (\ref{equ-ood})).
% \end{Theorem}
\begin{proof}
The proof is trivial. Since two sets $\cup_{e\in\mathcal{E}_{tr}}\mathcal{Z}^e_{spu}(Y^e=1)$ and $\cup_{e\in \mathcal{E}_{tr}}\mathcal{Z}^e_{spu}(Y^e=0)$ are linearly separable, there exists a lineal classifier $w$ that only relies on spurious features and can achieve zero classification error on each environment. Therefore, $w$ is an invariant predictor across different training environments. Also, $H(Z^e_{inv})>H(Z^e_{spu})$ would make IB-IRM prefer to choose these spurious features. Therefore, $w$ would be an optimal solution of IB-IRM, ERM, IRM, and IB-ERM. However, since $w$ relies on spurious features which may change arbitrary in unseen environments, it thus fails to solve OOD generalization.
\end{proof}

\subsection[\appendixname~\thesubsection]{Proof of Theorem \ref{theorem-CSIB}}

\begin{proof}
Assume $\Phi^*\in \mathbb{R}^{c\times d}$ and $w^*$ are the feature extractor and classifier learned by IB-ERM. Consider the feature variable extracted by $\Phi^*$ as
\begin{equation}
\Phi^* X^e = \Phi^* S(Z^e_{inv}, Z^e_{spu}) = \Phi_{inv}Z^e_{inv} + \Phi_{spu}Z^e_{spu}.
\end{equation}
We first show that $\Phi_{inv} = \bm{0}$ or $\Phi_{spu} = \bm{0}$. We prove this by contradiction. Assume $\Phi_{inv} \neq \bm{0}$ and $\Phi_{spu}\neq \bm{0}$. By observing that a solution of $\Phi_{inv}=\bm{1}, \Phi_{spu} = \bm{0}, w^*=w^*_{inv}$ could make the average training error to $q$, therefore any solution returned by IB-ERM should also achieve the error no larger than $q$ (because $r^{th} = q$ in the constraint of Equation \ref{equ-IBERM}). Therefore $w^*\neq \bm{0}$.
\begin{enumerate}
\item In the case when each $e\in \mathcal{E}_{tr}$ follows Assumption \ref{assp-mode-1} of $Z^e_{spu}\leftarrow AZ^e_{inv} + W^e$, we have
    \begin{gather*}
    w^*\cdot(\Phi_{inv}Z^e_{inv} + \Phi_{spu} Z^e_{spu}) = w^*\cdot\Phi_{inv}Z^e_{inv} + w^*\cdot\Phi_{spu} (AZ^e_{inv} + W^e) \\
    = w^*\cdot(\Phi_{inv} + \Phi_{spu}A)Z^e_{inv} + w^*\cdot\Phi_{spu}W^e.
    \end{gather*}
    Then, for any $z = (z^e_{inv}, z^e_{spu})$ of $\bm{1}(w^*_{inv}\cdot z^e_{inv}) = 1$, we must have $w^*\cdot(\Phi_{inv} + \Phi_{spu}A)z^e_{inv} + w^*\cdot\Phi_{spu}w^e \geq 0$ for any $w^e$ to make error no larger than $q$. Since $W^e$ is zero mean with at least two distinct points in each component, we can conclude that $w^*\cdot(\Phi_{inv} + \Phi_{spu}A)z^e_{inv} \geq 0$; Similarly, for any $z = (z^e_{inv}, z^e_{spu})$ of $\bm{1}(w^*_{inv}\cdot z^e_{inv}) = 0$, we have $w^*\cdot(\Phi_{inv} + \Phi_{spu}A)z^e_{inv} < 0$. From Lemma \ref{appendix-lemma-iib-3} or Lemma \ref{appendix-lemma-iib-4}, we get $H((\Phi_{inv} + \Phi_{spu}A)Z^e_{inv} + \Phi_{spu}W^e) > H((\Phi_{inv} + \Phi_{spu}A)Z^e_{inv})$. Therefore, there exists a more optimal solution to IB-ERM with zero weight to $Z^e_{spu}$, which contradicts the assumption.
\item In the case when each $e\in \mathcal{E}_{tr}$ follows Assumption \ref{assp-mode-2} of $Z^e_{inv}\leftarrow AZ^e_{spu} + W^e$, we have
    \begin{gather*}
    w^*\cdot(\Phi_{inv}Z^e_{inv} + \Phi_{spu} Z^e_{spu}) = w^*\cdot\Phi_{inv}(AZ^e_{spu} + W^e) + w^*\cdot\Phi_{spu}Z^e_{spu} \\
    = w^*\cdot(\Phi_{spu} + \Phi_{inv}A)Z^e_{spu} + w^*\cdot\Phi_{inv}W^e.
    \end{gather*}
    % Then, for any $z = (z^e_{inv}, z^e_{spu})$ of $\bm{1}(w^*_{inv}\cdot z^e_{inv}) = 1$, we must have $w^*\cdot(\Phi_{spu} + \Phi_{inv}A)z^e_{spu} + w^*\cdot\Phi_{inv}w^e \geq 0$ for any $w^e$ to make error no larger than $q$. Since $W^e$ is zero mean with at least two distinct points, we can conclude that $w^*\cdot(\Phi_{spu} + \Phi_{inv}A)z^e_{spu} \geq 0$; Similarly, for any $z = (z^e_{inv}, z^e_{spu})$ of $\bm{1}(w^*_{inv}\cdot z^e_{inv}) = 0$, we have $w^*\cdot(\Phi_{spu} + \Phi_{inv}A)z^e_{spu} < 0$.
    From Lemma \ref{appendix-lemma-iib-3} or Lemma \ref{appendix-lemma-iib-4}, we get $H((\Phi_{spu} + \Phi_{inv}A)Z^e_{spu} + \Phi_{inv}W^e)>H((\Phi_{spu} + \Phi_{inv}A)Z^e_{spu})$. In addition, the spurious features are assumed to be linearly separable. Therefore, there exists a more optimal solution to IB-ERM with zero weight to $Z^e_{inv}$, which contradicts the assumption.
\item In the case when each $e\in \mathcal{E}_{tr}$ follows Assumption \ref{assp-mode-3} of $Z^e_{spu} \leftarrow W_1^e Y^e + W_0^e (1-Y^e)$, we have
    \begin{gather*}
    w^*\cdot(\Phi_{inv}Z^e_{inv} + \Phi_{spu} Z^e_{spu}) = w^*\cdot\Phi_{inv}Z^e_{inv} + w^*\cdot\Phi_{spu} (W_1^e Y^e + W_0^e (1-Y^e)) \\
    = w^*\cdot\Phi_{inv}Z^e_{inv} + w^*\cdot\Phi_{spu} W_1^e Y^e + w^*\cdot\Phi_{spu}W^e_0(1-Y^e).
    \end{gather*}
    Then, for any $z = (z^e_{inv}, z^e_{spu})$ of $\bm{1}(w^*_{inv}\cdot z^e_{inv}) = 1$, we must have $w^*\cdot\Phi_{inv}z^e_{inv} + w^*\cdot\Phi_{spu} w_1^e y^e + w^*\cdot\Phi_{spu} w_0^e (1-y^e) \geq 0$ for any $w_1^e$ and $w_0^e$ to make error no larger than $q$. Since $W_1^e$ and $W_0^e$ are both zero mean variables with at least two distinct points in each component, we can conclude that $w^*\cdot\Phi_{inv}z^e_{inv} \geq 0$; Similarly, for any $z = (z^e_{inv}, z^e_{spu})$ of $\bm{1}(w^*_{inv}\cdot z^e_{inv}) = 0$, we have $w^*\cdot\Phi_{inv}z^e_{inv} < 0$. From Lemma \ref{appendix-lemma-iib-3} or Lemma \ref{appendix-lemma-iib-4}, we get $H(\Phi_{inv}Z^e_{inv} + \Phi_{spu} W_1^e Y^e + \Phi_{spu}W^e_0(1-Y^e)) > H(\Phi_{inv}Z^e_{inv})$. Therefore, there exists a more optimal solution to IB-ERM with zero weight to $Z^e_{spu}$, which contradicts the assumption.

\end{enumerate}
So far, we have proved that the feature extractor $\Phi^*$ learned by IB-ERM would never extract both spurious features and invariant features together. Then, we perform singular value decomposition (SVD) to the $\Phi^*$ as
\begin{align}
\Phi^* = U\Lambda V^T =  [U_1, U_2][\Lambda_1, \bm{0}; \bm{0}, \bm{0}] [V_1^T; V_2^T] = U_1\Lambda_1 V_1^T
\end{align}
Let $S\in \mathbb{R}^{d\times d}$ be the orthogonal matrix. Set $r$ be the rank of the matrix $\Phi^*$, i.e., $r=Rank(\Phi^*)$, and let $V_1^TS = [V'_1, V'_2]$ with $V'_1\in\mathbb{R}^{r\times m}$ and $V'_2\in\mathbb{R}^{r\times o}$, and $V_2^TS = [V''_1, V''_2]$ with $V''_1\in\mathbb{R}^{(d-r)\times m}$ and $V''_2\in\mathbb{R}^{(d-r)\times o}$, then
\begin{align}
\Phi^*X^e = U_1\Lambda_1 V_1^T S [Z^e_{inv};Z^e_{spu}] = U_1\Lambda_1(V'_1Z^e_{inv} + V'_2Z^e_{spu}).
\end{align}
Since $\Phi^*X^e$ contains the information either from spurious features or from invariant features, we must have $U_1\Lambda_1V'_1 = \bm{0}$ or $U_1\Lambda_1V'_2 = \bm{0}$, and thus, $V'_1 = \bm{0}$ or $V'_2 = \bm{0}$ due to $Rank(U_1\Lambda_1) = r$. If $V'_2 = \bm{0}$, then $\Phi^*$ extract invariant features only. Otherwise when $V'_1 = \bm{0}$, we decompose the $V^TS$ by
\begin{align}
V^TS = [V^T_1;V^T_2]S = [V^T_1S;V^T_2S] = [V'_1, V'_2; V''_1, V''_2].
\end{align}
Since $V^T$ and $S$ are both the orthogonal matrix, $V^TS$ is also orthogonal, thus $V'_1=\bm{0}\Rightarrow {V'}_2^T V''_2 = \bm{0}$, and then $Rank(V''_2) = Rank([V'_2;V''_2]) - Rank(V'_2) = o-r$ (note that $r\leq \min\{m,o\}$). Then,
\begin{align}
V^T_2 X^e = V^T_2 S [Z^e_{inv};Z^e_{spu}] = [V''_1, V''_2][Z^e_{inv};Z^e_{spu}]=V''_1Z^e_{inv} + V''_2Z^e_{spu}.
\end{align}
Therefore, by running the CSIB for one iteration, the rank of spurious features would be decreased by $r > 0$. This would result in zero weight to spurious features by finite runs of CSIB.

Then, we tend to show why the counterfactual supervision step could help to distinguish whether $V'_1$ is $\bm{0}$ or not. For a specific instance $x = S[z_{inv}; z_{spu}]$, let two new features be $z^1$ and $z^2$, then $do(z^1_{1:r}) = [-M, .., -M]$ and $do(z^1_{r+1:d})=V_2^T x$; $do(z^2_{1:r}) = [M, .., M]$ and $do(z^2_{r+1:d})=V_2^T x$.
Back the new features $z^1$ and $z^2$ to the input space as $x^1 = Vz^1$ and $x^2 = Vz^2$.
If $V'_1 =\bm{0}$, then
\begin{gather*}
S^{-1}x^1 = S^{-1}Vz^1 = S^{-1}V[z^1_{1:r}; V''_1z_{inv} + V''_2z_{spu}] \\
 = (V^TS)^T[z^1_{1:r}; V''_1z_{inv} + V''_2z_{spu}] \\
 = [V'^T_1, V''^T_1; V'^T_2, V''^T_2] [z^1_{1:r}; V''_1z_{inv} + V''_2z_{spu}] \\
 = [V'^T_1z^1_{1:r} + V''^T_1(V''_1z_{inv}+V''_2z_{spu}); V'^T_2z^1_{1:r} + V''^T_2(V''_1z_{inv}+V''_2z_{spu})] \\
 = [z_{inv};V'^T_2z^1_{1:r} + V''^T_2 V''_2 z_{spu}],
\end{gather*}
and similar we have $S^{-1}x^2 = [z_{inv};V'^T_2z^2_{1:r} + V''^T_2 V''_2 z_{spu}]$. Therefore, the ground truths of $x^1$ and $x^2$ are the same. On other hand, if $V'_1\neq \bm{0}$, then $V'_2 = \bm{0}$, and
\begin{gather*}
S^{-1}x^1 = S^{-1}Vz^1 = S^{-1}V[z^1_{1:r}; V''_1z_{inv} + V''_2z_{spu}] \\
 = (V^TS)^T[z^1_{1:r}; V''_1z_{inv} + V''_2z_{spu}] \\
 = [V'^T_1, V''^T_1; V'^T_2, V''^T_2] [z^1_{1:r}; V''_1z_{inv} + V''_2z_{spu}] \\
 = [V'^T_1z^1_{1:r} + V''^T_1(V''_1z_{inv}+V''_2z_{spu}); V'^T_2z^1_{1:r} + V''^T_2(V''_1z_{inv}+V''_2z_{spu})] \\
 = [V'^T_1z^1_{1:r}+ V''^T_1V''_1z_{inv};z_{spu}],
\end{gather*}
and similar we have $S^{-1}x^2 = [V'^T_1z^2_{1:r}+ V''^T_1V''_1z_{inv};z_{spu}]$. Since $z^1_{1:r} = -z^2_{1:r}$ and their magnitudes are larger enough to make $\textbf{sgn}(w^*_{inv}\cdot(V'^T_1z^1_{1:r}+ V''^T_1V''_1z_{inv}))\neq\textbf{sgn}(w^*_{inv}\cdot(V'^T_1z^2_{1:r}+ V''^T_1V''_1z_{inv}))$, thus the ground truths of $x^1$ and $x^2$ would be different. Therefore, the counterfactual supervision step could help to detect whether invariant features or spurious features are extracted by using a single sample only.

Finally, when only invariant features are extracted by $\Phi$, the training error is minimized, i.e., $w^*\Phi_{inv} \in \arg\min_{f} \mathbb{E}_{\mathbb{P}}[l(f(Z^{tr}_{inv}), Y^{tr})]$. Then, based on our assumption to the OOD environments (Assumptions \ref{assp-new-assp}), i.e., $\forall e\in\mathcal{E}_{ood}, F_l(\mathbb{P}(Z^{tr}_{inv}, Y^{tr}))\subseteq F_l(\mathbb{P}(Z^e_{inv},Y^e))$, therefore, for any $e\in\mathcal{E}_{ood}$, we have $\mathbb{E}_{\mathbb{P}}[l((X^e, Y^e), w^*\Phi)] = \mathbb{E}_{\mathbb{P}}[l((Z^e_{inv}, Y^e), w^*\Phi_{inv})] = \mathbb{E}_{\mathbb{P}}[l((Z^{tr}_{inv}, Y^{tr}), w^*\Phi_{inv})] = q$.
\end{proof}

It is worth to note that the proof of Theorem \ref{theorem-CSIB} does not rely on how many labels there would be, so it is easily extended to the multi-class classification case as long as the corresponding assumptions and conditions are satisfied.

\end{document}